\documentclass{article}

\usepackage{microtype}
\usepackage{graphicx}
\usepackage{subfigure}
\usepackage{booktabs} 

\usepackage{hyperref}

\usepackage[accepted]{icml2025}

\usepackage{amsmath}
\usepackage{amssymb}
\usepackage{mathtools}
\usepackage{amsthm}

\usepackage[capitalize,noabbrev]{cleveref}

\theoremstyle{plain}
\newtheorem{theorem}{Theorem}[section]
\newtheorem{proposition}[theorem]{Proposition}

\newtheorem{corollary}[theorem]{Corollary}
\theoremstyle{definition}
\newtheorem{definition}[theorem]{Definition}
\newtheorem{assumption}[theorem]{Assumption}
\theoremstyle{remark}

\newcommand{\loss}[1]{\mathcal{L}(#1)}

\usepackage[textsize=tiny]{todonotes}

\icmltitlerunning{Why Domain Generalization Fail? A View of Necessity and Sufficiency}

\begin{document}

\twocolumn[
\icmltitle{Why Domain Generalization Fail? A View of Necessity and Sufficiency}



\icmlsetsymbol{equal}{*}

\begin{icmlauthorlist}
\icmlauthor{Tung-Long vuong}{comp}
\icmlauthor{Vy Vo}{comp}
\icmlauthor{Hien Dang}{yyy}
\icmlauthor{Van-Anh Nguyen}{comp}
\icmlauthor{Thanh-Toan Do}{comp}
\icmlauthor{Mehrtash Harandi}{comp}
\icmlauthor{Trung Le}{comp}
\icmlauthor{Dinh Phung}{comp}

\end{icmlauthorlist}

\icmlaffiliation{yyy}{ University of Texas at Austin, USA}
\icmlaffiliation{comp}{Monash University, Australia}

\icmlcorrespondingauthor{Tung-Long Vuong}{Tung-Long.Vuong@monash.edu}

\icmlkeywords{Machine Learning, ICML}

\vskip 0.3in
]




\printAffiliationsAndNotice{} 

\begin{abstract}

Despite a strong theoretical foundation, empirical experiments reveal that existing domain generalization (DG) algorithms often fail to consistently outperform the ERM baseline. We argue that this issue arises because most DG studies focus on establishing theoretical guarantees for generalization under unrealistic assumptions, such as the availability of sufficient, diverse (or even infinite) domains or access to target domain knowledge. As a result, the extent to which domain generalization is achievable in scenarios with limited domains remains largely unexplored. This paper seeks to address this gap by examining generalization through the lens of the conditions necessary for its existence and learnability. Specifically, we systematically establish a set of necessary and sufficient conditions for generalization. Our analysis highlights that existing DG methods primarily act as regularization mechanisms focused on satisfying sufficient conditions, while often neglecting necessary ones. However, sufficient conditions cannot be verified in settings with limited training domains. In such cases, regularization targeting sufficient conditions aims to maximize the likelihood of generalization, whereas regularization targeting necessary conditions ensures its existence. Using this analysis, we reveal the shortcomings of existing DG algorithms by showing that, while they promote sufficient conditions, they inadvertently violate necessary conditions. To validate our theoretical insights, we propose a practical method that promotes the \textit{sufficient condition} while maintaining the \textit{necessary conditions} through a novel subspace representation alignment strategy. This approach highlights the advantages of preserving the necessary conditions on well-established DG benchmarks.

\end{abstract}

\section{Introduction}\label{sec:main_intro}

Domain generalization (DG) aims to train a machine learning model on multiple data distributions so that it can generalize to unseen data distributions. Although challenging, DG is crucial for practical scenarios where there is a need to quickly deploy a prediction model on a new target domain without access to target data. Various approaches have been proposed to address the DG problem, which can be broadly categorized into $4$ families: representation alignment, invariant prediction, augmentation and  ensemble learning. 

\textit{Representation alignment} which is established based on domain adaptation \citep{ben2010theory, ben2001support, phung2021learning, zhou2020deep, johansson2019support}, mainly discuss the differences between source and target domains and focuses on learning domain-invariant representations by reducing the divergence between latent marginal distributions \citep{long2017conditional, ganin2016domain, nguyen2021domain, shen2018wasserstein, xie2017controllable, ilse2020diva} or class-conditional distributions \citep{gong2016domain, li2018deep, tachet2020domain}. 


\textit{Invariant prediction} approaches, grounded in causality theory, ensure stable performance regardless of the domain by learning a consistently optimal classifier \citep{arjovsky2020irm, ahuja2020empirical, krueger2021out, li2022invariant,mitrovic2020representation, zhang2023causal}. 


\textit{Data augmentation} applies predefined or learnable transformations on the original samples or their features to create augmented data, thereby enhancing the model's generalization capabilities \citep{mitrovic2020representation, wang2022out, zhou2020deep, zhou2021domain, zhang2017mixup, wang2020heterogeneous, zhao2020maximum, yao2022improving, carluccidomain, yao2022pcl}.  With sufficient and diverse causal-preserving transformations, the model's generalization can be effectively guaranteed \citep{wang2020heterogeneous}.

Recently, \textit{Ensemble Learning}, which explicitly involves training multiple instances of the same architecture with different initializations or splits of the training data then combining their predictions \citep{zhou2021domain, ding2017deep, zhou2021domain, wang2020dofe, mancini2018best, cha2021swad, arpit2022ensemble}. Alternatively, implicit ensemble methods that approximate ensembling by averaging model weights (WA) from training trajectories (e.g., checkpoints at different time steps) have been shown to significantly enhance robustness under domain shifts \citep{izmailov2018averaging, cha2021swad, rame2022diverse, wortsman2022robust}.
Despite their strong performance in DG, most theoretical analyses of ensemble learning remain concentrated on the perspective of flatness of the loss landscape, or uncertainty-aware frameworks, leaving the connection between ensemble and DG largely underexplored.

Although the conventional DG algorithms (\textit{representation alignment, invariant prediction, data augmentation}) are developed with strong theoretical foundations, these methods have not consistently outperformed Empirical Risk Minimization (ERM) on fair model selection criteria \citep{gulrajani2020search, idrissi2022simple, ye2022ood, chen2022does}. In contrast, \textit{Ensemble}-based approaches (e.g., SWAD \citep{cha2021swad})  demonstrate a substantial performance improvement over the ERM  and other conventional DG algorithms by a large margin. We argue that this phenomenon arises because the theoretical frameworks behind conventional DG algorithms typically establish generalization under conditions where target domains are either known or sufficiently diverse, or when a large number of training domains are available. As a result, the extent to which domain generalization can be achieved in practical scenarios, characterized by a limited and finite number of domains, remains elusive unexplored. 
\begin{table*}[t]
\caption{Summary of Conditions for Generalization}
\begin{centering}
\resizebox{\linewidth}{!}{ %
\begin{tabular}{llll}
\toprule
\textbf{Condition} & \textbf{Type}  & \textbf{Target DG approach}\\
\midrule
Label-identifiability (\ref{as:label_idf}) &Assumption \\
\midrule
Causal support (\ref{as:sufficient_causal_support}) &  Assumption  \\
\midrule
Optimal hypothesis for training domains (\ref{def:joint_optimal})& Necessary  & \\
\midrule
Optimal hypothesis for training domains $+$ Invariant representation function (\ref{thm:sufficient_conditions}.1) &  Sufficient & Representation alignment \\
\midrule
Optimal hypothesis for training domains $+$ Sufficient and diverse domains (\ref{thm:sufficient_conditions}.2) & Sufficient & Invariant prediction \\
\midrule
Optimal hypothesis for training domains $+$ Invariance-preserving transformations (\ref{thm:sufficient_conditions}.3) & Sufficient & Data augmentation \\
\midrule
\textbf{Invariance-preserving representation function} (\ref{def:sufficient}) & \textbf{Necessary} & Ensembles \\
\bottomrule
\end{tabular}}
\par\end{centering}
\label{tab:conditions}
\end{table*}
Our work aims to fill in this gap with a comprehensive study of DG in scenarios with limited training domains through the lens of necessary and sufficient conditions for achieving generalization. The contributions of this paper are summarized as follows:

\textbf{1. DG through the lens of Necessity and Sufficiency.} We systematically establish a set of necessary and sufficient conditions for generalization, highlighting that existing DG methods act as regularization mechanisms that predominantly focus on satisfying sufficient conditions while often neglecting necessary ones (see Section \ref{sec:main_conds}.1). However, sufficient conditions are non-verifiable with limited training domains. In such cases, we argue that regularization targeting sufficient conditions aims to maximize the likelihood of generalization, while regularization targeting necessary conditions ensures its existence (see Section \ref{sec:main_conds}.2). 


\textbf{2. Why do conventional DG algorithms fail?} In Section~\ref{sec:discussion_DG}, we explain why conventional DG algorithms fail to consistently outperform the ERM baseline by demonstrating that, while they promote sufficient conditions, they may inadvertently violate necessary conditions.
We also establish a connection between the necessary conditions and the recent \textit{ensemble} strategy, demonstrating that \textit{ensemble} methods indeed encourage models to satisfy these conditions.

  
\textbf{3. Preserving Necessity in Representation Alignment.} Finally, we empirically validate our theories by introducing a practical method that encourages the \textit{sufficient condition} \underline{without violating} the \textit{necessary conditions} through a novel subspace representation alignment strategy. Our approach achieves superior performance across all experimental settings (see Section~\ref{sec:main_proposed_method}).



\section{Preliminaries}
We first introduce the notations and basic concepts in the paper. We use calligraphic letters (i.e., $\mathcal{X}$) for spaces, upper case letters (i.e. $X$) for random variables, lower case letters (i.e. $x$) for their values, $\mathbb{P}$ for probability distributions and $\text{supp}(\cdot)$ specifies the support set of a distribution.

\subsection{Revisiting Domain Generalization setting}


We consider a standard domain generalization setting with a potentially high-dimensional variable $X$ (e.g., an image), a label variable $Y$ and a discrete environment (or domain)
variable $E$ in the sample spaces $\mathcal{X}, \mathcal{Y}$, and $\mathcal{E}$, respectively. Specifically, we focus on a multi-class classification problem with the label set $\mathcal{Y}=\left[C\right]$,
where $C$ is the number of classes and $\left[C\right]:=\{1,\ldots,C\}$. Denote $\mathcal{Y}_{\Delta}:=\left\{ \alpha\in\mathbb{R}^{C}:\left \| \alpha \right \|_{1}=1\,\wedge\,\alpha\geq 0\right\} $
as the $C-$simplex, let $f:\mathcal{X}\mapsto\mathcal{Y}_{\Delta}$
be a hypothesis predicting a $C$-tuple $f\left(x\right)=\left[f\left(x\right)[i]\right]_{i=1}^{C}$,
whose element $f\left(x\right)[i]=p\left(y=i\mid x\right)$,  
is the probability to assign a data sample $x\sim\mathbb{P}$
to the class $i$ with $i\in\left\{ 1,...,C\right\} $. Let $l:\mathcal{Y}_{\Delta}\times\mathcal{Y}\mapsto\mathbb{R}$ be
a loss function, where $l\left(f\left(x\right),y\right)$ with $f\left(x\right)\in\mathcal{Y}_{\Delta}$ and $y\in\mathcal{Y}$
specifies the loss (e.g., cross-entropy, Hinge, L1, or L2 loss) to
assign a data sample $x$ to the class $y$ by the hypothesis $f$. The general 
loss of the hypothesis $f$ w.r.t. a given domain $\mathbb{P}^e$ is:
\begin{equation}
\mathcal{L}\left(f,\mathbb{P}^e\right):=\mathbb{E}_{\left(x,y\right)\sim\mathbb{P}^e}\left[\ell\left(f\left(x\right),y\right)\right].   
\end{equation}

\color{black}

\textit{Objective}: Given a set of training domains $\mathcal{E}_{tr}=\{e_1,...,e_K\} \subset \mathcal{E}$, the objective of DG is to exploit the `commonalities' present in the training domains to improve generalization to any domain of the population $e\in \mathcal{E}$. For supervised classification, the task is equivalent to seeking the set of \textbf{global optimal hypotheses} $\mathcal{F}^{*}$ 
where every $f\in \mathcal{F}^*$ is locally optimal for every domain:
    \begin{equation}
    \mathcal{F}^{*} := \bigcap_{{e}\in \mathcal{E}}\underset{f\in \mathcal{F}}{\text{argmin}} \ \loss{f,\mathbb{P}^{e}}
    \label{eq:optimal}
\end{equation}



\color{black}

We here examine the widely used {\it composite hypothesis} $f = h \circ g \in \mathcal{F}$, where $g : \mathcal{X} \rightarrow \mathcal{Z}$ belongs to a set of representation functions $\mathcal{G}$, mapping the data space $\mathcal{X}$ to a latent space $\mathcal{Z}$, and $h : \mathcal{Z} \rightarrow \mathcal{Y}_{\Delta}$ is the classifier in the space $\mathcal{H}$. 

\textbf{Presumption}. While our work considers limited and a finite number of domains, we follow recent theoretical works \citep{wang2022provable, rosenfeld2020risks, kamath2021does, ahuja2021invariance, chen2022iterative} assuming the infinite data setting for every training environment. This assumption distinguishes DG literature from traditional generalization analysis (e.g., PAC-Bayes framework) that focuses on in-distribution generalization where the testing data are drawn from the same distribution.

\subsection{Assumptions on Data Generation Process}
We consider the following family of distributions over the observed variables $(X,Y)$ given the environment $E=e \in \mathcal{E}$ where environment space under consideration $\mathcal{E} = \{ e \mid \mathbb{P}^e \in \mathcal{P} \}$:
\vspace{-2mm}
\begin{equation*}
    \mathcal{P}=\left \{\mathbb{P}^e(X, Y)=\int_{z_c}\int_{z_e}\mathbb{P}(X, Y, Z_c,Z_e, E=e)d z_c d z_e\right \}
\end{equation*}
The data generative process underlying every observed distribution $\mathbb{P}^e(X, Y)$ is characterized by a \textit{structural causal model} (SCM) over a tuple  $\left \langle V, U, \psi \right \rangle$ (See Figure~\ref{fig:graph}). The SCM consists of a set of \textit{endogenous} variables $V = \{X, Y, Z_c, Z_e, E\}$, a set of mutually independent \textit{exogenous} variables $U = \{U_x, U_y, U_{z_c}, U_{z_e}, U_e\}$ associated with each variable in $V$ and a set of deterministic equations $\psi = \{\psi_x, \psi_y, \psi_{z_c}, \psi_{z_e}, \psi_e\}$ representing the generative process for $V$. We note that this generative structure has been widely used and extended in several other studies, including \citep{chang2020invariant, mahajan2021domain, li2022invariant, zhang2023causal, lu2021invariant,liu2021heterogeneous}. 

The generative process begins with the sampling of an environmental variable $e$ from a prior distribution $\mathbb{P}(U_e)$\footnote{explicitly via the equation $e = \psi_e(u_e), u_e \sim P(U_e)$.}. We assume there exists a causal factor $z_c\in\mathcal{Z}_c$ determining the label $Y$ and a environmental feature $z_e\in\mathcal{Z}_e$ \textit{spuriously} correlated with $Y$. 
These two latent factors are 
generated from an environment $e$ via the mechanisms $z_c = \psi_{z_c}(e, u_{z_c})$ and $z_e = \psi_{z_e}(e, u_{z_e})$ with $u_{z_c} \sim \mathbb{P}(U_{z_c}), u_{z_e} \sim \mathbb{P}(U_{z_e})$. A data sample $x\in \mathcal{X}$ is generated from both the causal feature and the environmental feature i.e., $x = \psi_{x}(z_c, z_e, u_{x})$ with $u_x \sim \mathbb{P}(U_x)$.

\begin{figure}
    \centering
    \includegraphics[width=0.5\linewidth]{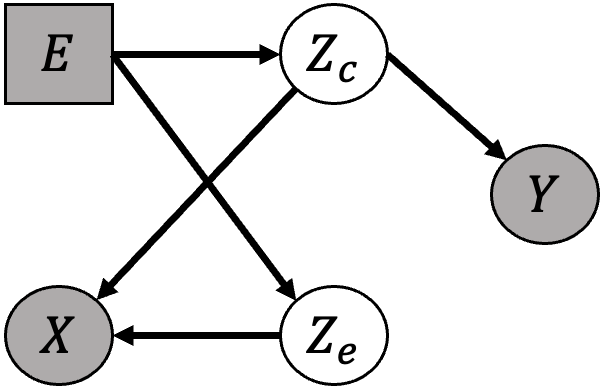}
    \caption{A directed acyclic graph (DAG) describing the causal relations among different factors producing data $X$ and label $Y$ in our SCM. Observed variables are shaded.}
    \label{fig:graph}
\end{figure}



Figure \ref{fig:graph} dictates that the joint distribution over $X$ and $Y$ can vary across domains resulting from the variations in the distributions of $Z_c$ and $Z_e$. Furthermore, \textit{both causal and environmental features are correlated with $Y$, but only $Z_c$ causally influences $Y$}.  However, because $Y \perp\!\!\!\perp E | Z_c$, the conditional distribution of $Y$ given a specific $Z_c = z_c$ remains unchanged across different domains i.e., $\mathbb{P}^e(Y | Z_c = z_c) = \mathbb{P}^{e'}(Y | Z_c = z_c)$ $\forall e, e' \in \mathcal{E}$. For readability, we omit the superscript $e$ and denote this invariant conditional distribution as $\mathbb{P}(Y | Z_c = z_c)$. 
\subsection{Assumptions on Possibility of Generalization}

We first establish crucial assumptions for the feasibility of generalization as described in Eq (\ref{eq:optimal}). These assumptions are essential for understanding the conditions under which generalization can be achieved.


\begin{assumption} (Label-identifiability). We assume that for any pair $z_c, z^{'}_c\in \mathcal{Z}_c$,  $\mathbb{P}(Y|Z_c=z_c) = \mathbb{P}(Y|Z_c=z^{'}_c) \text{ if } \psi_x(z_c,z_e,u_x)=\psi_x(z_c',z'_e,u'_x)$ for some $z_e, z'_e, u_x, u'_x$
\label{as:label_idf}.
\end{assumption}

The causal graph indicates that $Y$ is influenced by $z_c$, making $Y$ identifiable over the distribution $\mathbb{P}(Z_c)$. This assumption implies that different causal factors $z_c$ and $z^{'}_c$ cannot yield the same $x$, unless the condition $\mathbb{P}(Y|Z_c=z_c) = \mathbb{P}(Y|Z_c=z^{'}_c)$ holds, or  the distribution $\mathbb{P}(Y\mid x)$ is stable. 

Assumption \ref{as:label_idf} gives rise to a family of invariant representation function $\mathcal{G}_c$, as stated in Proposition \ref{thm:invariant_correlation} below. 

\begin{proposition} (Invariant Representation Function)
Under Assumption.\ref{as:label_idf}, there exists a set of deterministic representation function $(\mathcal{G}_c\neq \emptyset)\in \mathcal{G}$ such that for any $g\in \mathcal{G}_c$, $\mathbb{P}(Y\mid g(x)) = \mathbb{P}(Y\mid z_c)$ and $g(x)=g(x')$ holds true for all $\{(x,x',z_c)\mid  x= \psi_x(z_c, z_e, u_x), x'= \psi_x(z_c, z^{'}_e, u^{'}_x) \text{ for all }z_e,z^{'}_e, u_x, u^{'}_x\}$ (Appendix \ref{thm:invariant_correlation_apd}).   
\label{thm:invariant_correlation}
\end{proposition}

Assumption \ref{as:label_idf} and Proposition~\ref{thm:invariant_correlation} license the existence of global optimal hypotheses as defined in Eq. (\ref{eq:optimal}).



\begin{assumption} (Causal support). We assume that the union of the support of \textit{causal} factors across training domains covers the entire causal factor space $\mathcal{Z}_c$, i.e., $\cup_{e\in \mathcal{E}_{tr}}\text{supp}\{\mathbb{P}^{e} \left (Z_c \right )\}=\mathcal{Z}_c$. 
\label{as:sufficient_causal_support}
\end{assumption}


This assumption holds significance in DG theories \citep{johansson2019support, ruan2021optimal, li2022sparse}, especially when we avoid imposing strict constraints on the target functions. Particularly, \citep{ahuja2021invariance} showed that without the support overlap assumption on the causal features, OOD generalization is impossible for such a simple linear setting. Meanwhile, for more complicated tasks, deep neural networks are typically employed. However, when trained via gradient descent, they cannot effectively approximate a broad spectrum of nonlinear functions beyond their support range \citep{xu2020neural}. It is worth noting that causal support overlap does not imply that the distribution over the causal features is held unchanged.

\section{DG: A view of Necessity and
Sufficiency}\label{sec:main_conds}

In this section, we present the necessary and sufficient conditions for achieving generalization defined in Eq. (\ref{eq:optimal}) (See Table \ref{tab:conditions} for summary). These conditions are critical to our analysis, where we first reveal that the existing DG methods aim to satisfy one or several of these necessary and sufficient conditions to achieve generalization. 

\subsection{Conditions for Generalization}

We begin with the definition of the global optimal hypothesis in Eq. (\ref{eq:optimal}), where $f\in \mathcal{F}^*$
  must also be the optimal solution across all domains. Thus, $f$ must first satisfy the requirement of being the optimal hypothesis for all training domains, which can be considered a \textit{necessary condition} for generalization, formally stated as follows:

\begin{definition} \textit{(Necessary-Condition-1: Optimal hypothesis for training domains) Given  $\mathcal{F}_{\mathbb{P}^e}=\underset{f\in \mathcal{F}}{\text{argmin}} \ \loss{f,\mathbb{P}^{e}}$ is set of optimal hypothesis for $\mathbb{P}^{e}$, the optimal hypothesis for all training domains is defined as $f\in\mathcal{F}_{\mathcal{E}_{tr}} = \bigcap_{{e}\in \mathcal{E}_{tr}}\mathcal{F}_{\mathbb{P}^e}$.}
\label{def:joint_optimal}
\end{definition}

However, $f \in \mathcal{F}_{\mathcal{E}_{tr}}$ is not sufficient to guarantee that $f \in \mathcal{F}^*$ as a global optimal hypothesis. This drives the exploration of analyses and algorithms that identify the conditions, objectives or constraints required to truly achieve generalization. The following theorem highlights that conventional DG algorithms predominantly focus on satisfying sufficient conditions for generalization.


\begin{theorem} (Sufficient conditions) Under Assumption \ref{as:label_idf} and Assumption \ref{as:sufficient_causal_support}, given a hypothesis $f=h\circ g$, if $f$ is optimal hypothesis for training domains i.e.,
\begin{equation*}
    f\in \bigcap_{{e}\in \mathcal{E}_{tr}}\underset{f\in \mathcal{F}}{\text{argmin}} \ \loss{f,\mathbb{P}^{e}}
\end{equation*}
and one of the following sub-conditions holds:
\begin{enumerate}
    \item $g$ belongs to the set of \textbf{invariant representation functions} as specified in Proposition~\ref{thm:invariant_correlation}.
    
    \item $\mathcal{E}_{tr}$ is a set of \textbf{Sufficient and diverse training domains} i.e., the union of the support of joint causal and spurious factors across training domains covers the entire causal and spurious factor space $\mathcal{Z}_c\times\mathcal{Z}_e$ i.e., $\cup_{e\in \mathcal{E}_{tr}}\text{supp}\{\mathbb{P}^{e} \left (Z_c, Z_e \right )\}=\mathcal{Z}_c\times\mathcal{Z}_e$.
    
    \item Given $\mathcal{T}$ is set of all \textbf{invariance-preserving transformations} such that for any $T\in \mathcal{T}$ and $g_c\in \mathcal{G}_c$: $(g_c\circ T)(\cdot)=g_c(\cdot)$, $f$ is also an optimal hypothesis on all augmented domains i.e., $$f\in \bigcap_{{e}\in \mathcal{E}_{tr}, T\in \mathcal{T}}\underset{f\in \mathcal{F}}{\text{argmin}} \ \loss{f,T\#\mathbb{P}^{e}}$$
\end{enumerate}
Then $f\in \mathcal{F}^*$ (Proof is in Appendix \ref{thm:sufficient_conditions_apd})
\label{thm:sufficient_conditions}.
\end{theorem}

Theorem \ref{thm:sufficient_conditions} states that the Necessary-Condition-1 ``optimal hypothesis for training domains" condition is the primary objective to achieve generalization and should not be violated when performing DG algorithms. The other three sub-conditions, \textit{Invariant Representation Function, Sufficient and Diverse Training Domains, and Invariance-Preserving Transformations}, are additional properties, or \textit{``constraints"}, to transfer the optimality on training domains to the true generalization on unseen domains. The theorem demonstrates that the main objective, combined with any one of these sub-conditions, forms a sufficient condition for the solution $f$ to generalize. 

Note that the high-level findings corresponding to each sufficient condition in theorem above are not new. However, generalizing these results within our framework in the form of sufficient conditions provides a fresh perspective on the functioning of DG algorithms, enabling a deeper analysis of when and why they succeed or fail.  Particularly, each sub-condition corresponds to a conventional DG strategy, as follows:
\begin{itemize}
    \item \textit{Invariant Representation Function}: This is the core objective of the \textit{representation alignment} methods \cite{ben2010theory, lu2021invariant,zhang2023causal}.
    
    \item \textit{Sufficient and Diverse Training Domains}: This constraint is independent of specific algorithms but appears as the backbone in the analysis of \textit{invariant prediction}-based algorithms \citep{ ahuja2021invariance}.
    
    \item \textit{Invariance-Preserving Transformations}: This ensures generalization for the family of \textit{augmentation}-based DG algorithms \citep{mitrovic2020representation,gao2023out}.
\end{itemize}

However, satisfying these sub-conditions (or constraints) is often impractical in scenarios with a limited number of training domains. In practice, these constraints act as regularization mechanisms, shaping the feasible optimal hypothesis space. To gain deeper insights into the dynamics of the hypothesis space under such limitations, we shift our focus to another necessary condition, \textit{the invariance-preserving representation function}, which is crucial yet frequently overlooked in the DG literature. It is formally defined as follows:





\begin{definition}\textit{(Necessary-Condition-2: Invariance-preserving representation function) A set of representation functions $\mathcal{G}_s\subset \mathcal{G}$ is considered as invariance-preserving representation functions if for any $g\in\mathcal{G}_s$,
$I(g(X),g_c(X))=I(X,g_c(X)$ where $g_c$ is an invariant representation function (i.e., $g(X)$ retains all the information about the invariant features of $g_c(X)$ from $X$). }
\label{def:sufficient}
\end{definition}

Recall that an invariant representation function seeks to extract \textit{exact} invariant representations that remain consistent across all environments. In contrast, an invariance-preserving representation function captures representations that contain invariant information, along with potentially other information, rather than strictly enforcing invariance. 

\begin{theorem} Given a representation function $g$,
$\exists h: h\circ g\in \mathcal{F}^*$ if and only if $g\in \mathcal{G}_s$. (Proof is in Appendix~\ref{thm:nacessary_apd})
\label{thm:nacessary}
\end{theorem}

This theorem implies that if $g$ is not an invariance-preserving representation function, i.e., $g \notin \mathcal{G}_s$, no classifier $h$ can exist such that $f=h \circ g \in \mathcal{F}^*$. In other words, $g \in \mathcal{G}_s$ is a \textit{necessary} condition for $f \in \mathcal{F}^*$. This property is critical for understanding the generalization ability of DG algorithms, as discussed in the following section.

\subsection{Why conventional DG algorithms fail?}
\label{sec:efficacy_DG}

\begin{figure}[h!]
    \centering
    \includegraphics[width=1.0\linewidth]{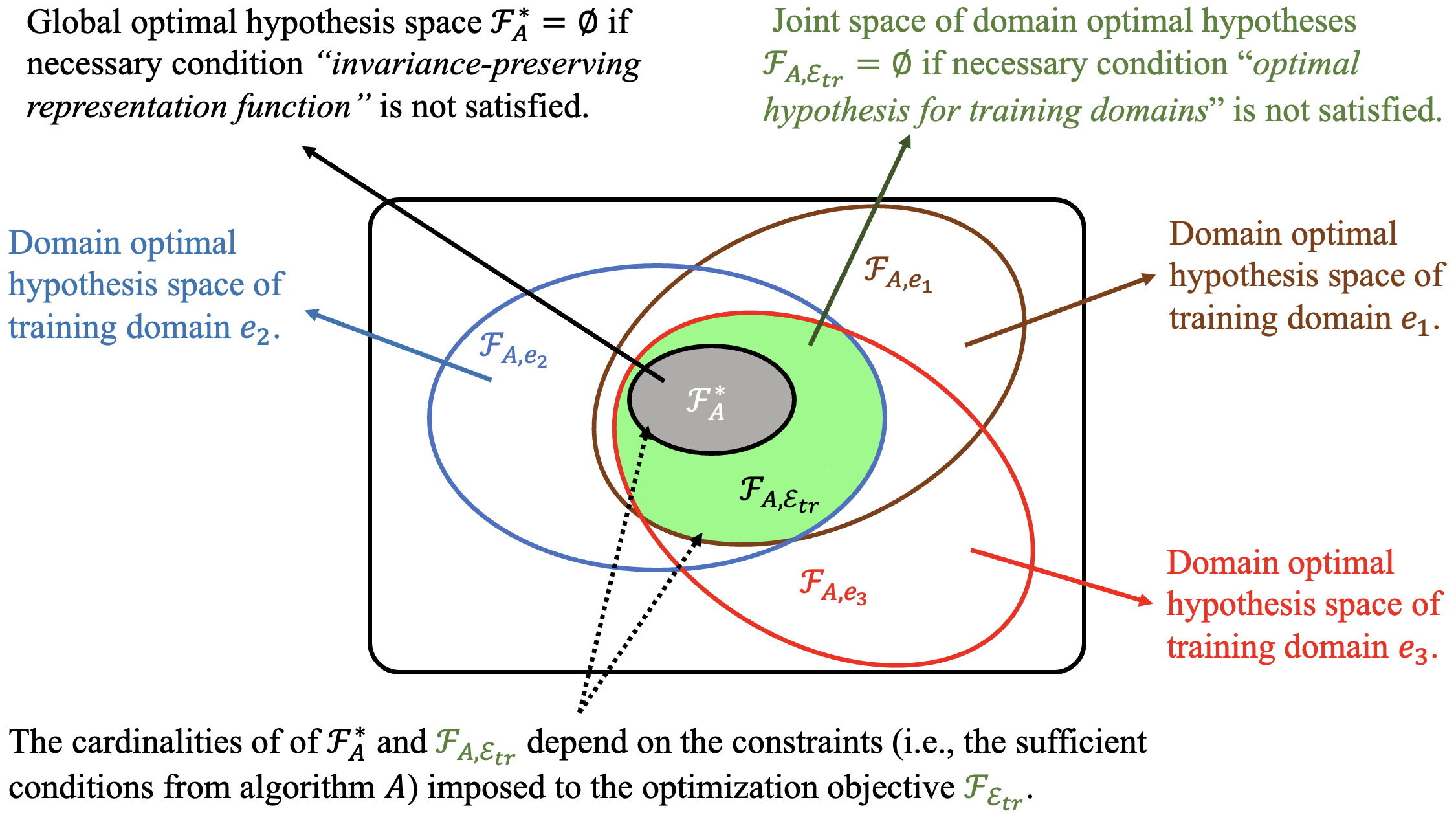}
    \vspace{-7mm}
    \caption{Venn diagram of the optimal hypothesis spaces induced by a DG algorithm $A$.}
    \vspace{-7mm}
\label{fig:space}
\end{figure}

In this Section, we illustrate the relationship between sufficient and necessary conditions using Venn diagrams of hypothesis spaces corresponding to each condition to understand how DG algorithms help generalization when sufficient conditions cannot be fully met. Additionally, we denote the feasible hypothesis spaces induced by applying a DG algorithm $A$ as $\mathcal{F}^*_{A},\mathcal{F}_{A,e_i}, \mathcal{F}_{A,\mathcal{E}_{tr}}$.

We start with two necessary conditions. By definition of the first necessary condition, "Optimal Hypothesis for Training Domains," we have 
$\mathcal{F}^*_A \subseteq \mathcal{F}_{A,\mathcal{E}_{tr}}$. Thus, in Figure~\ref{fig:space}, the global optimal hypothesis space $\mathcal{F}^*_A$ (grey area) lies within the joint domain-optimal hypothesis space $\mathcal{F}_{A,\mathcal{E}_{tr}}$ (green area). According to Theorem~\ref{thm:nacessary}, if the "Invariance-Preserving Representation Function" condition is not satisfied, the feasible global optimal hypothesis space becomes empty $\mathcal{F}^*_A=\emptyset$. In other words, satisfying the Necessary-Condition-2 ensures the existence of generalization.

Turning to sufficient conditions, Theorem~\ref{thm:sufficient_conditions} states that if these conditions (or constraints) are satisfied, a global optimal hypothesis is achieved, meaning the green area converges to the grey area. Otherwise, the constraints imposed by conventional DG algorithms act as regularization on the first necessary condition, primarily aiming to shrink the green area. If the grey area remains intact, reducing the green area increases the likelihood of achieving generalization. However, while restricting the set of feasible joint optimal hypotheses for the training domains (green area), these regularizations may inadvertently shrink the grey area. If the constraints are arbitrary or overly restrictive, there is a risk that the grey area reduces to null, ultimately causing the DG algorithms to fail.



\section{Understanding DG Literature via Necessity} 
 \label{sec:discussion_DG}

In the previous section, we examined the role of sufficient conditions and the importance of necessary conditions in ensuring generalization. Building on this, we analyze why existing DG methods fail by highlighting that while they promote sufficient conditions, they violate necessary ones.

\subsection{Representation Alignment} 

Representation Alignment focus on learning domain-invariant representations by reducing the divergence between latent marginal distributions $\mathbb{E}[g(X) | E]$ where $E$ represents a domain environment. Other methods seek to align the conditional distributions $\mathbb{E}[g(X) | Y=y, E]$ across domains. Achieving true invariance is inherently challenging and may impose overly restrictive conditions. In certain cases, better feature alignment can result in higher joint errors, as demonstrated in the following theorem:

\begin{theorem}
\label{theorem:single_tradeoff} \citep{zhao2019learning, phung2021learning, le2021lamda} Distance between two marginal distribution $\mathbb{P}^{e}_\mathcal{Y}$ and $\mathbb{P}^{e'}_\mathcal{Y}$ can be upper-bounded: 
\vspace{-1mm}
\begin{equation*}
\begin{aligned}
D\left(\mathbb{P}_{\mathcal{Y}}^{e},\mathbb{P}_{\mathcal{Y}}^{e'}\right) \leq 
D\left ( g_{\#}\mathbb{P}^{e},g_{\#}\mathbb{P}^{e'} \right )
+\mathcal{L}\left ( f,\mathbb{P}^{e} \right )
+
\mathcal{L}\left ( f,\mathbb{P}^{e'} \right )
\end{aligned}
\end{equation*}
where $g_{\#}\mathbb{P}(X)$ denotes representation distribution on  representation space $\mathcal{Z}$ induced by applying encoder with $g: \mathcal{X} \mapsto \mathcal{Z}$ on data distribution $\mathbb{P}$. $D$ can be $\mathcal{H}$-divergence \citep{zhao2019learning}, Hellinger distance \citep{phung2021learning} or Wasserstein distance \citep{le2021lamda} (see Appendix~\ref{apd:tradeoff}).
\end{theorem}

Theorem~\ref{theorem:single_tradeoff} suggests that if there is a substantial discrepancy in the label marginal distribution $D(\mathbb{P}_{\mathcal{Y}}^{e},\mathbb{P}_{\mathcal{Y}}^{e'})$ across training domains, strongly enforcing \textit{representation alignment} $D( g_{\#}\mathbb{P}^{e},g_{\#}\mathbb{P}^{e'})$ will lead to an increase in \textit{domain-losses} $\left ( \mathcal{L}( f,\mathbb{P}^{e}) + \mathcal{L}( f,\mathbb{P}^{e'})\right )$. In other words, while the representation alignment strategy promotes the development of an invariant representation function, it may also pose a challenge to Necessary-Condition-1. 
\label{sec:representation_alignment}

\subsection{Augmentation} 
\begin{figure}[h!]
    \centering
\includegraphics[width=1.0\linewidth]{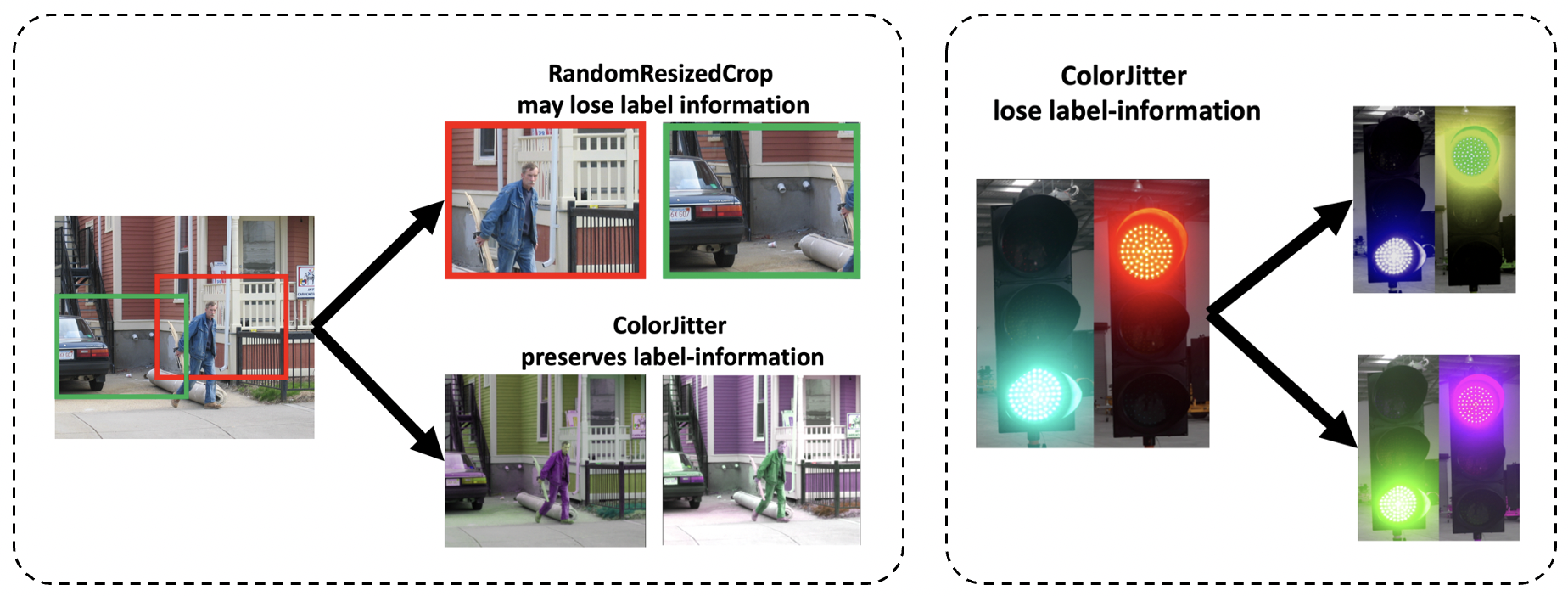}
\vspace{-7mm}
    \caption{(Left) RandomResizedCrop alters the label-information, whereas ColorJitter does not. (Right) ColorJitter modifies the label-information of traffic light.}
    \label{fig:aug}
    \vspace{-4mm}
\end{figure}
Data augmentation utilizes predefined or learnable transformations $T$ on the original sample $X$ or its features $g(x)$ to create augmented data $T(X)$ or $T(g(x))$. However, it's crucial that transformation $T$ maintains the integrity of the causal factors. For instance, in Figure~\ref{fig:aug}.Left, random-cropping augmentation alters the causal features, whereas color-based augmentation does not. However, in Figure~\ref{fig:aug}.Right, color-based augmentation modifies the causal features. This implies a \textit{necessity for some knowledge of the target domain} to ensure the transformations do not alter the label information \citep{gao2023out,zhang2022rethinking}, otherwise it risks violating Necessary-Condition-2.
\label{sec:augmentation}

\subsection{Invariant Prediction} These methods aim to learn a consistent optimal classifier across domains. For example, Invariant Risk Minimization (IRM) \citep{arjovsky2020irm} seeks to learn a representation function $g(x)$ with invariant predictors $\mathbb{E}[Y | g(x), E]$. This goal aligns with Necessary-Condition-1 without imposing restrictions that could affect Necessary-Condition-2. In fact, in fully informative invariant features setting (i.e., $Y \!\perp\! X\mid g(x)$) or the number of training domains is limited, IRM does not provide significant advantages over ERM \citep{rosenfeld2020risks, ahuja2020empirical, ahuja2021invariance}. VREx \citep{krueger2021out} relaxes the IRM's constraint to enforce equal risks across domains, assuming that the optimal risks are similar across domains. If, however, the optimal solutions exhibit large loss variations, balancing risks could result in suboptimal performance for some domains, violating Necessary-Condition-1.
IB-IRM \citep{ahuja2021invariance} posits that integrating the information bottleneck (minimizing $I(g(X); X)$) principle to invariant prediction strategies aids generalization.

However, information bottleneck approach is beneficial only when there is a sufficient and diverse set of training domains \cite{ahuja2021invariance}. Otherwise, the information bottleneck may lead to a higher risk of violating Necessary-Condition-2. The following corollary demonstrates our arguments.


\begin{corollary} Under Assumption \ref{as:label_idf} and Assumption \ref{as:sufficient_causal_support}, let the minimal representation function $g_{\text{min}}$ be defined as:
\begin{equation*}
g_{\text{min}} \in \left\{\underset{g \in \mathcal{G}}{\text{argmin }} I(g(X); X) \ \text{s.t.} \ f = h \circ g \in \mathcal{F}_{\mathcal{E}_{\text{tr}}} \right\},
\end{equation*}
where $I(\cdot,\cdot)$ denotes mutual information. Then, for any $g_c\in \mathcal{G}_c$ the following holds:
\begin{equation}
I(g_{\text{min}}(X), g_c(X)) \leq I(X, g_c(X)),
\label{eq:ineq}
\end{equation}
and the equality holds if and only if at least one of sufficient conditions is hold. (Proof is provided in Appendix \ref{thm:information_apd}.)
\label{thm:information}.
\end{corollary}

This corollary, derived from Theorem~\ref{thm:sufficient_conditions}, shows that the minimal representation \( g_{\text{min}} \) violates Necessary-Condition-2  i.e., $I(g_{\text{min}}(X), g_c(X)) < I(X, g_c(X))$ unless one of the sufficient conditions is met, in that case, the equality holds.

\begin{figure}[h!]
    \centering
\includegraphics[width=1.0\linewidth]{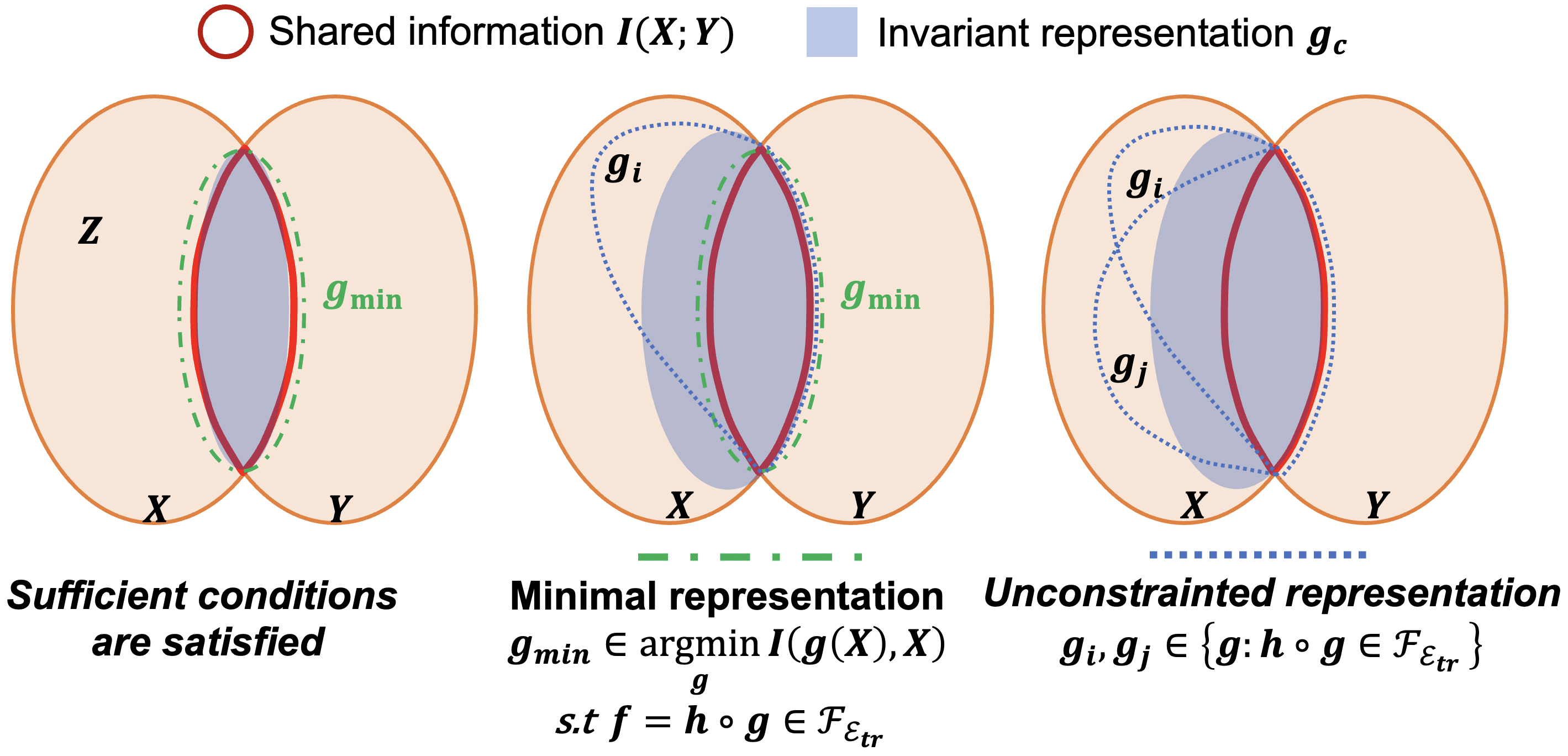}
\vspace{-7mm}
    \caption{Information diagrams of \( X, Y \); the invariant representation \( g_c(X) \); the minimal representation \( g_{\text{min}}(X) \); and the representations \( g_i(X), g_j(X) \), where there exist corresponding classifiers on top of these representations that form optimal hypotheses for the training domains.}
    \label{fig:info_min}
\end{figure}

Particularly, the intuition of this corollary is illustrated in Figure~\ref{fig:info_min}.Mid which depicts the mutual information among the four variables $X,Y,g_c(X)$ and $g_{min}(X)$. By Assumption~\ref{as:sufficient_causal_support} and Proposition~\ref{thm:invariant_correlation}, it follows that $X$ contains all information about $Z_c$. Based on the data generation process, we have $I(X,Y \mid Z_c)=0$, indicating that the causal features $Z_c$ must encapsulate the shared information $I(X;Y)$. 
By Theorem~\ref{thm:sufficient_conditions}, if one sufficient condition is satisfied, generalization is achieved with $g_{min}$, which implies: $I(g_{min}(X), g_c(X)) = I(X, g_c(X))$ (Figure~\ref{fig:info_min}.Left).
Otherwise, \( g_{\text{min}} \) tends to favor simple features, such as textures or backgrounds, rather than more complex and generalizable semantic features, like shapes \citep{geirhos2020shortcut}, leading to violate Necessary-Condition-2 (Figure~\ref{fig:info_min}.Right).

\subsection{Ensemble Learning} 
\label{sec:sufficient_constraint}

In line with the analyses based on the information-theoretic perspective in the previous section, we intuitively discuss a connection between the Necessary-Condition-2 and the recent \textit{ensemble} strategy, demonstrating that \textit{ensemble} methods indeed encourage models to satisfy this. Particularly, it can be observed that any representation \( g_i \) capable of forming an optimal hypothesis for the training domains captures the shared information \( I(X; Y) \) and potentially some additional information about \( g_c(X) \) (Figure~\ref{fig:info_min} (Mid)).

One possible way to increase the likelihood of \( g_i(X) \) capturing information about \( g_c(X) \) is to maximize \( I(X, g_i(X)) \). However, this approach may not be appropriate when combined with domain generalization (DG) algorithms designed to encourage sufficient conditions. For instance, representation alignment approaches aim to retain shared information across domains while removing other information, whereas maximizing \( I(X, g_i(X)) \) seeks to preserve as much information about \( X \) as possible, regardless of its relevance.

In contrast, \textit{Ensemble Learning} can encourage the representation to capture more information from \( g_c(X) \) by learning multiple diverse versions of representations through ensemble methods to capture as much information as possible about \( g_c(X) \) (as illustrated in Figure~\ref{fig:info_min} (Right)). Additionally, each version \( g_i(X), \dots, g_j(X) \) can also align with the constraints necessary to satisfy the sufficient conditions.



\section{Necessity-Preserving DG Algorithm}
\label{sec:main_proposed_method}

Section~\ref{sec:discussion_DG} highlights that while conventional DG algorithms promote sufficient conditions, they often violate necessary conditions, which, as analyzed in Section~\ref{sec:main_conds}, is the primary reason for their failure to generalize. To further validate our analysis, we propose a method that promote sufficient conditions without violate necessary conditions.

To recap, the three conventional strategies are as follows: \textit{Invariant prediction}-based methods require sufficient and diverse domains, which depend heavily on the data;
\textit{Augmentation}-based methods rely on invariance-preserving transformations, which often require access to the target domain;
\textit{Representation alignment}-based methods are more flexible and can be directly manipulated. Therefore, in this section, we present a method that imposes the representation alignment constraint without violating necessary conditions, thereby offering a practical means to verify our analysis.

\subsection{Subspace Representation Alignment (SRA)}

As demonstrated in Theorem~\ref{theorem:single_tradeoff}, \textit{Representation Alignment}-based methods face a trade-off between alignment constraints and domain losses due to discrepancies in label distribution across domains, resulting in a violation of necessary conditions. The following theorem demonstrates that by appropriately organizing training domains into distinct subspaces, where differences in marginal label distributions are removed, we can align representations within these subspaces without violating necessary conditions.

\begin{theorem}
\label{theorem:multi_bound} \textit{(Subspace Representation Alignment)} Given \textit{subspace projector} $\Gamma: \mathcal{X}\rightarrow \mathcal{M}$, and a subspace index $m\in \mathcal{M}$, let $A_{m}=\Gamma^{-1}(m)=\left\{ x:\Gamma(x)=m\right\} $ is the region on data space which has the same index $m$ and $\mathbb{P}_{m}^{e}$ be the distribution restricted by $\mathbb{P}^{e}$ over the set $A_{m}$, then $\pi^{e}_m=\frac{\mathbb{P}^{e}\left(A_{m}\right)}{\sum_{m'\in\mathcal{M}}\mathbb{P}^{e}\left(A_{m'}\right)}$ is mixture co-efficient, if the loss function $\ell$ is upper-bounded by a positive
constant $L$, then:

(i)  The target general loss is upper-bounded: 
\begin{align*}
\left | \mathcal{E}_{tr} \right |\sum_{e\in \mathcal{E}_{tr}}\mathcal{L}\left ( f,\mathbb{P}^{e} \right )
&\leq
\sum_{e\in \mathcal{E}_{tr}} \sum_{m\in\mathcal{M}}\pi^{e}_m
\mathcal{L}\left ( f,\mathbb{P}^{e}_{m} \right ) \\+
&L\sum_{e, e'\in \mathcal{E}_{tr}}\sum_{m\in\mathcal{M}}\pi^{e}_{m}D\left ( g_{\#}\mathbb{P}^{e}_{m},g_{\#}\mathbb{P}^{e'}_{m} \right ),
\end{align*}
(ii) Distance between two label marginal distribution $\mathbb{P}^{e}_{m}(Y)$ and $\mathbb{P}^{e'}_{m}(Y)$ can be upper-bounded: 
\begin{equation*}
\begin{aligned}
D\left(\mathbb{P}^{e}_{\mathcal{Y},m},\mathbb{P}^{e'}_{\mathcal{Y},m}\right) &\leq 
D\left ( g_{\#}\mathbb{P}^{e}_{m},g_{\#}\mathbb{P}^{e'}_{m} \right )\\
+&\mathcal{L}\left ( f,\mathbb{P}^{e}_{m}\right )
+
\mathcal{L}\left ( f,\mathbb{P}^{e'}_{m} \right )
\end{aligned}
\end{equation*}
(iii) Construct the subspace projector $\Gamma$ as the optimal hypothesis for the training domains i,e., \(\Gamma \in \mathcal{F}_{\mathcal{E}_{tr}}\), which defines $\mathcal{M}=\{m=\hat{y}\mid \hat{y}=\Gamma(x), x\in\bigcup_{e\in\mathcal{E}_{tr}}\text{supp}\mathbb{P}^{e} \}\subseteq \mathcal{Y}_\Delta$, then
$
D\left(\mathbb{P}^{e}_{\mathcal{Y}, m},\mathbb{P}^{e'}_{\mathcal{Y},m}\right)=0$ for all \(m \in \mathcal{M}\).

where $g_{\#}\mathbb{P}$ denotes representation distribution on $\mathcal{Z}$ induce by applying $g$ with $g: \mathcal{X} \mapsto \mathcal{Z}$ on data distribution $\mathbb{P}$, $D$ can be $\mathcal{H}$-divergence, Hellinger or Wasserstein distance. 
\end{theorem}

Theorem \ref{theorem:multi_bound}.(i) shows that the \textit{domain losses} on the left-hand side (LHS) are upper-bounded by the \textit{subspace losses} and \textit{subspace alignments} within individual subspaces on the right-hand side (RHS). 
If the RHS is perfectly optimized (i.e., without trade-offs), representation alignment can be imposed without violating the necessary conditions.

Theorem \ref{theorem:multi_bound}.(ii) demonstrates that the distance between the marginal label distributions is now defined within subspaces, denoted as \( D\left(\mathbb{P}^{e}_{\mathcal{Y},m},\mathbb{P}^{e'}_{\mathcal{Y},m}\right) \). It is worth noting that \( D\left(\mathbb{P}^{e}_{\mathcal{Y},m},\mathbb{P}^{e'}_{\mathcal{Y},m}\right) \) can be adjusted based on how \(\Gamma\) distributes the training domains across subspaces. 

Theorem \ref{theorem:multi_bound}.(iii) shows that by constructing the subspace projector \(\Gamma\) as the optimal hypothesis for the training domains i.e., \(\Gamma \in \mathcal{F}_{\mathcal{E}_{tr}}\), we achieve $D\left(\mathbb{P}^{e}_{\mathcal{Y},m}, \mathbb{P}^{e'}_{\mathcal{Y},m}\right) = 0$
for all \(m \in \mathcal{M}\), allowing us to jointly optimize both \textit{subspace losses} and \textit{subspace alignments} without trade-off.


\textbf{SRA objective.} Building on Theorem \ref{theorem:multi_bound}, the objective of the subspace representation alignment algorithm is defined as:
\begin{align}
&\min_{f=h\circ g} \underset{\text{Subspace Representation Alignment}}{\underbrace{\sum_{e,e'\in \mathcal{E}_{tr}}\sum_{m\in \mathcal{M}}D\left( g\#\mathbb{P}_m^{e},g\#\mathbb{P}_m^{e'}\right)}}\label{eq:final_objective}\\
&\text{ s.t. } \underset{\text{Training domain optimal hypothesis}}{\underbrace{f=h\circ g\in \bigcap_{e\in \mathcal{E}_{tr}}\underset{ f}{\text{argmin }} \mathcal{L}\left(f,\mathbb{P}^{e}\right)}} \nonumber
\end{align}


where $\mathcal{M}=\{\hat{y}\mid \hat{y}=f(x), x\in\bigcup_{e\in\mathcal{E}_{tr}}\text{supp}\mathbb{P}^{e} \}$ and $D$ can be $\mathcal{H}$-divergence, Hellinger distance, Wasserstein distance. We provide the details on the practical implementation of the proposed objective in Appendix~\ref{Sec:practical}.

\subsection{Experimental results}\label{sec:main_exp}
\begin{table}[h!]
\caption{Classification accuracy (\%) across datasets.
}
\begin{centering}
\resizebox{1.0\columnwidth}{!}{ %
\begin{tabular}{lcccccc}
\toprule
\textbf{Algorithm}  & \textbf{VLCS} & \textbf{PACS} & \textbf{OfficeHome} & \textbf{TerraInc}  & \textbf{DomainNet} & \textbf{Avg} \\
\toprule
ERM~\citep{gulrajani2020search}  & 73.5  & 85.5  & 66.5  & 46.1  & 40.9  & 63.3 \\
IRM~\citep{arjovsky2020irm} & 78.5 & 83.5 & 64.3 & 47.6 & 33.9 & 61.6 \\
VREx~\citep{krueger2021out} & 78.3 & 84.9 & 66.4 & 46.4 & 33.6 & 61.9 \\
IB-IRM~\citep{ahuja2021invariance} & 77.7 & 85.8 & 70,9  &43.4  & 35.1 & 62.6  \\
DANN~\citep{ganin2016domain} & 78.6  & 83.6  & 65.9  & 46.7  & 38.3  & 62.6 \\
CDANN~\citep{li2018domain}  & 77.5  & 82.6  & 65.8  & 45.8  & 38.3  & 62.0 \\
\textbf{Ours} (SRA)  & 76.4  & 86.3  & 66.4  & 49.5  & 44.5  & 64.6 \\
\midrule
SWAD~\citep{cha2021swad} & 79.1  & 88.1  & 70.6  & 50.0   &46.5  & 66.9\\
SWAD + IRM & 78.8 & 88.1 & 70.4 & 49.6 & 39.9 & 65.4\\
SWAD + VRE & 78.1 & 85.4 & 69.9 & 50.0 & 40.0 &64.9\\
SWAD + IB-IRM~ & 78.8 & 88.1 & 70.9 & 49.3 & 38.4  & 65.1 \\
SWAD + DANN & 79.2  & 87.9  & 70.5  & 50.6   &45.7  & 66.8\\
SWAD + CDANN & 79.3  & 87.7  & 70.4  & 50.7   &45.7  & 66.8\\
\textbf{Ours} (SRA + SWAD) & \underline{79.4}  & \underline{88.7}  &  \underline{72.1}  &  \underline{51.6}  & \underline{47.6}   & \underline{67.9} \\
\midrule
SRA + SWAD + Ensemble & \textbf{79.8}  & \textbf{89.2}  &  \textbf{73.2}  &  \textbf{52.2}  & \textbf{48.7}   & \textbf{68.6} \\
\bottomrule
\end{tabular}}
\par\end{centering}
\label{tab:Averages}
\end{table}
We present empirical evidence from experiments conducted on 5 datasets from the DomainBed benchmark (full experimental results and detailed settings are provided in Appendix~\ref{apd:settings}) to validate our theoretical analyses:

\textit{(i) Conventional DG algorithms fail because they inadvertently violate necessary conditions:} As shown in Table~\ref{tab:Averages}, none of the baseline methods consistently surpass ERM overall, regardless of whether SWAD is applied. In contrast, our proposed approach SRA consistently outperforms all baseline methods across all datasets.  

Additionally, we highlight that SRA is most similar to DANN and CDANN. Like these methods, SRA utilizes $\mathcal{H}$-divergence for alignment; however, the key distinction lies in the alignment strategy:  DANN aligns the entire domain representation,  CDANN aligns class-conditional representations, while SRA employs subspace-conditional alignment.
As discussed in Theorem~\ref{theorem:single_tradeoff}, DANN and CDANN potentially violate necessary conditions, whereas SRA does not (Theorem~\ref{theorem:multi_bound}). The superior performance of SRA compared to DANN and CDANN further confirms this.

\textit{(ii) Ensemble-based approaches promote the "Invariance-Preserving Representation Function" Condition, leading to improve generalization:} 
The results of the baselines with SWAD in Table~\ref{tab:Averages} highlight the crucial benefits of using ensembles to encourage Necessary-Condition-2 for improved generalization. To further demonstrate the advantages of an ensemble strategy, we average the predictions of models trained with different random seeds (SRA + SWAD + Ensemble), leading to a noticeable performance boost.

\begin{table}[h!]
\vspace{-5mm}
\caption{Classification accuracy on PACS using ResNet-50 with varying information bottleneck coefficient $\lambda$.}
\begin{centering}
\resizebox{0.8\columnwidth}{!}{ %
\begin{tabular}{lccccc}
\toprule
$\lambda$  & \textbf{Art\_painting} & \textbf{Cartoon} & \textbf{Photo} & \textbf{Sketch} & \textbf{Average}  \\
\midrule
1 & 89.8 & 82.4 & 97.7 & 82.6 & 88.1\\
10 & 87.9 & 82.2 & 97.2 & 78.5 & 86.5\\
100 & 85.6 & 75.9 & 97.5 & 44.1 & 75.8\\
\bottomrule
\end{tabular}}
\par\end{centering}
\label{tab:infor}
\vspace{-2mm}
\end{table}
\textit{(iii) The Information-Bottleneck regularization potentially violating Necessary-Condition-2, being ineffective for generalization:}  As shown in Table~\ref{tab:infor}, the results indicate that the information bottleneck strategy performs poorly in practical scenarios, particularly in the IB-IRM baseline. Increasing the information bottleneck regularization strength can lead to significant performance degradation, especially when the source domains (e.g., Photo, Art, Cartoon) contain rich label information while the target domain ("Sketch") has limited label information (e.g., $\lambda=100$, the default value in IB-IRM). In contrast, when the target domain is "Photo" or "Art," the model achieves relatively better performance.
\vspace{-1mm}

\section{Conclusion}\label{sec:main_conclusion}

This paper provides a fresh perspective on existing DG algorithms in the context of limited training domains, analyzed through the lens of necessary and sufficient conditions for generalization. Our analysis reveals that the failure of conventional DG algorithms arises from their focus on promoting sufficient conditions while neglecting and often inadvertently violating necessary conditions. Furthermore, we provide new insights into two recent strategies, ensemble learning and information bottleneck. The success of ensemble learning lies in its promotion of the necessary condition of the "Invariance-Preserving Representation Function." In contrast, the information bottleneck approach proves ineffective for generalization as it violates this condition, contradicting findings from previous research.

\section*{Impact Statement}

``This paper presents work whose goal is to advance the field of 
Machine Learning. There are many potential societal consequences 
of our work, none which we feel must be specifically highlighted here.''

\bibliography{references}
\bibliographystyle{icml2025}

\newpage
\appendix
\onecolumn
\section{Theoretical development}
\label{apd:proof}
In this section, we present all the proofs  of our theoretical development. 

For readers' convenience, we recapitulate our definition and assumptions:

\textit{Domain objective}: Given a domain $\mathbb{P}^e$, let the hypothesis $f:\mathcal{X}\rightarrow\Delta_{\left | \mathcal{Y} \right |}$ is a map from the data space $\mathcal{X}$ to the the $C$-simplex label space $\Delta_{\left | \mathcal{Y} \right |}:=\left\{ \alpha\in\mathbb{R}^{\left | \mathcal{Y} \right |}:\left \| \alpha \right \|_{1}=1\,\land\,\alpha\geq 0\right\}$.
Let $l:\mathcal{Y}_{\Delta}\times\mathcal{Y}\mapsto\mathbb{R}$ be a loss function, where $\ell\left(f\left(x\right),y\right)$ with
$f\left(x\right)\in\mathcal{Y}_{\Delta}$ and $y\in\mathcal{Y}$
specifies the loss (i.e., cross-entropy) to
assign a data sample $x$ to the class $y$ by the hypothesis $f$. The general 
loss of the hypothesis $f$ w.r.t. a given domain $\mathbb{P}^e$ is:
\begin{equation}
\mathcal{L}\left(f,\mathbb{P}^e\right):=\mathbb{E}_{\left(x,y\right)\sim\mathbb{P}^e}\left[\ell\left(f\left(x\right),y\right)\right].   
\end{equation}


\begin{assumption} (Label-identifiability). We assume that for any pair $z_c, z^{'}_c\in \mathcal{Z}_c$,  $\mathbb{P}(Y|Z_c=z_c) = \mathbb{P}(Y|Z_c=z^{'}_c) \text{ if } \psi_x(z_c,z_e,u_x)=\psi_x(z_c',z'_e,u'_x)$ for some $z_e, z'_e, u_x, u'_x$
\label{as:label_idf_apd}.
\end{assumption}

\begin{assumption} (Causal support). We assume that the union of the support of causal factors across training domains covers the entire causal factor space $\mathcal{Z}_c$: $\cup_{e\in \mathcal{E}_{tr}}\text{supp}\{\mathbb{P}^{e} \left (Z_c \right )\}=\mathcal{Z}_c$ where $\text{supp}(\cdot)$ specifies the support set of a distribution. 
\label{as:sufficient_causal_support_apd}
\end{assumption}

\begin{corollary}
    $\mathcal{F}\neq \emptyset$ if and only if Assumption~\ref{as:label_idf_apd} holds.
    \label{thm:existence_apd}
\end{corollary}

\begin{proof}

    The \textbf{"if"} direction is directly derived from the Proposition~\ref{thm:invariant_correlation_apd}.  We prove \textbf{"only if"} direction by contraction.

    If Assumption~\ref{as:label_idf_apd} does not hold, there a pair $x=x'$ such that $x= \psi_x(z_c,z_e,u_x)$ $x'=\psi_x(z_c',z'_e,u'_x)$ for some $z_e, z'_e, u_x, u'_x$ and $\mathbb{P}(Y|Z_c=z_c) \neq \mathbb{P}(Y|Z_c=z^{'}_c)$.

    By definition of $f\in\mathcal{F}^*$, $f(x)=\mathbb{P}(Y|Z_c=z_c)\neq \mathbb{P}(Y|Z_c=z^{'}_c)=f(x')=f(x)$ which is a contradiction. (It is worth noting that a domain containing only one sample $x$ is also valid within our data-generation process depicted in Figure~\ref{fig:graph}.).
\end{proof}

\begin{proposition} (Invariant Representation Function)
Under Assumption.\ref{as:label_idf_apd}, there exists a set of deterministic representation function $(\mathcal{G}_c\neq \emptyset)\in \mathcal{G}$ such that for any $g\in \mathcal{G}_c$, $\mathbb{P}(Y\mid g(x)) = \mathbb{P}(Y\mid z_c)$ and $g(x)=g(x')$ holds true for all $\{(x,x',z_c)\mid  x= \psi_x(z_c, z_e, u_x), x'= \psi_x(z_c, z^{'}_e, u^{'}_x) \text{ for all }z_e,z^{'}_e, u_x, u^{'}_x\}$
\label{thm:invariant_correlation_apd}
\end{proposition}




\begin{proof}
Under Assumption.\ref{as:label_idf_apd}, we can always choose a deterministic function $g_c: \mathcal{X}\rightarrow \mathcal{Z}_c$ such that the outcome of $g_c(x)$, can be any $z_c\in\{z_c\mid x= \psi_x(z_c, z_e, u_x)\}$ and $\mathbb{P}(Y\mid g_c(x))=\mathbb{P}(Y\mid z_c)$, will consistently provide an accurate prediction of $Y$. In essence, Y is identifiable over the pushforward measure $g_c\#\mathbb{P}(X)$.  



\end{proof}

\begin{corollary} \label{cor:proterties}(Invariant Representation Function Properties) For any \(g \in \mathcal{G}_c\), the following properties hold:
\begin{enumerate}
    \item(Causal representation:) \(g\) is a mapping function directly from the sample space \(\mathcal{X}\) to the causal feature space \(\mathcal{Z}_c\), such that \(g: \mathcal{X} \rightarrow \mathcal{Z}_c\).
    \item (Equivalent causal representation) Given a deterministic equivalent causal transformation mapping \(T: \mathcal{Z}_c \rightarrow \mathcal{Z}_c\), which maps a causal factor \(z_c\) to another equivalent causal factor \(T(z_c)\), such that
 $\mathbb{P}(Y\mid z_c)=\mathbb{P}(Y\mid T(z_c))$, then we have \(g(x) = T(z_c)\) holds for all \(\{x \mid x = \psi_x(z_c, z_e, u_x), \text{ for all } z_e, u_x\}\).

\item Given $\ell$ is the Cross-Entropy Loss i.e., $\ell(h(z_c), y) = -\sum_{y \in \mathcal{Y}} \mathbb{P}(Y = y \mid z_c) \log h(z_c)[y]$, there exists $h^*$ such that:
\begin{equation*}
  h^* \in\bigcap_{z_c\in\mathcal{Z}_c} \underset{h\in \mathcal{H}}{\text{argmin }} \mathbb{E}_{y\sim\mathbb{P}(Y\mid z_c)} \ell\left ( h( z_c), y \right ),  
\end{equation*}
\end{enumerate}
\end{corollary}
\color{black}
\begin{proof}
We prove each property as follows:


\underline{\textit{Proof of property-1:}} Suppose there exists $g: \mathcal{X}\rightarrow \mathcal{Z}$ such that $\mathbb{P}(Y\mid g(x)) = \mathbb{P}(Y\mid z_c)$ holds true for all $\{(x,z_c)\mid  x= \psi_x(z_c, z_e, u_x) \text{ for all }z_e,u_x\}$.

If \( g \) is not a function from \( \mathcal{X} \) to \( \mathcal{Z}_c \), then \( g(x) \) may include spurious features \( z_e \), or both \( z_c \) and \( z_e \) for $x=\psi(z_c,z_e,u_x)$.

Based on the structural causal model (SCM) depicted in Figure~\ref{fig:graph}, it follows that \( Z_e \not\!\perp\!\!\!\perp Y \), meaning that the environmental feature \( Z_e \) is spuriously correlated with \( Y \). Consequently, 
\[
\mathbb{P}(Y \mid g(x = \psi(z_c, z_e, u_x))) \neq \mathbb{P}(Y \mid g(x = \psi(z_c, z_e', u_x)))
\]
for some \( z_e \neq z_e' \), which is a contradiction.

\underline{\textit{Proof of property-2:}} Since $g: \mathcal{X}\rightarrow \mathcal{Z}_c$ and $\mathbb{P}(Y\mid g(x)) = \mathbb{P}(Y\mid z_c)$ holds true for all $\{(x,z_c)\mid  x= \psi_x(z_c, z_e, u_x) \text{ for all }z_e,u_x\}$, the outcome of $g(x)$ have to be any $z^{'}_c\in \mathcal{Z}_c$ such that $\mathbb{P}(Y\mid z_c)=\mathbb{P}(Y\mid z^{'})$, which means \(g(x) = T(z_c)\) holds for  \(\{x \mid x = \psi_x(z_c, z_e, u_x)\}\)

This highlights the flexibility of the family of invariant representation functions \(\mathcal{G}_c\), as they allow the model to map a sample \(x = \psi(z_c, z_e, u_x)\) to a set of equivalent causal factors \(\{z'_c \in \mathcal{Z}_c \mid \mathbb{P}(Y \mid z_c) = \mathbb{P}(Y \mid z'_c)\}\), rather than requiring an exact mapping to \(z_c\).

Finally, since $g(x)=g(x')$ holds true for all $\{(x,x',z_c)\mid  x= \psi_x(z_c, z_e, u_x), x'= \psi_x(z_c, z^{'}_e, u^{'}_x) \text{ for all }z_e,z^{'}_e, u_x, u^{'}_x\}$, \(g(x) = T(z_c)\) holds for all \(\{x \mid x = \psi_x(z_c, z_e, u_x), \text{ for all } z_e, u_x\}\)

\underline{\textit{Proof of property-3:}}

Given $z_c\in \mathcal{Z}_c$ and $\ell(h(z_c), y) = -\sum_{y \in \mathcal{Y}} \mathbb{P}(Y = y \mid z_c) \log h(z_c)[y]$, it is easy to show that the optimal $$h^*=\underset{h\in \mathcal{H}}{\text{argmin }} \mathbb{E}_{y\sim\mathbb{P}(Y\mid z_c)} \ell\left ( h( z_c), y \right )$$ is the conditional probability distribution $h^*(z_c)=\mathbb{P}(Y\mid z_c)$.

Based on structural causal model (SCM) depicted in Figure~\ref{fig:graph}, $\mathbb{P}(Y\mid z_c)$  remains stable across all domains. Therefore, there exists an optimal function \(h^*\) such that:
\begin{equation*}
  h^* \in\bigcap_{z_c\in\mathcal{Z}_c} \underset{h\in \mathcal{H}}{\text{argmin }} \mathbb{E}_{y\sim\mathbb{P}(Y\mid z_c)} \ell\left ( h( z_c), y \right ),  
\end{equation*}

where $h^*(z_c)=\mathbb{P}(Y\mid z_c)$ for all $z_c\in \mathcal{Z}_c$
\end{proof}

\subsection{Sufficient Conditions for achieving Generalization}

\begin{theorem} (\textbf{Theorem \ref{thm:sufficient_conditions} in the main paper}) Under Assumption \ref{as:label_idf} and Assumption \ref{as:sufficient_causal_support}, given a hypothesis $f=h\circ g$, if $f$ is optimal hypothesis for training domains i.e.,
\begin{equation*}
    f\in \bigcap_{{e}\in \mathcal{E}_{tr}}\underset{f\in \mathcal{F}}{\text{argmin}} \ \loss{f,\mathbb{P}^{e}}
\end{equation*}
and one of the following sub-conditions holds:
\begin{enumerate}
    \item $g$ belongs to the set of \textbf{invariant representation functions} as specified in Proposition~\ref{thm:invariant_correlation}.
    
    \item $\mathcal{E}_{tr}$ is set of \textbf{Sufficient and diverse training domains} i.e., the union of the support of joint causal and spurious factors across training domains covers the entire causal and spurious factor space $\mathcal{Z}_c\times\mathcal{Z}_e$ i.e., $\cup_{e\in \mathcal{E}_{tr}}\text{supp}\{\mathbb{P}^{e} \left (Z_c, Z_e \right )\}=\mathcal{Z}_c\times\mathcal{Z}_e$.
    
    \item Given $\mathcal{T}$ is set of all \textbf{invariance-preserving transformations} such that for any $T\in \mathcal{T}$ and $g_c\in \mathcal{G}_c$: $(g_c\circ T)(\cdot)=g_c(\cdot)$, $f$ is also an optimal hypothesis on all augmented domains i.e., $$f\in \bigcap_{{e}\in \mathcal{E}_{tr}, T\in \mathcal{T}}\underset{f\in \mathcal{F}}{\text{argmin}} \ \loss{f,T\#\mathbb{P}^{e}}$$
\end{enumerate}
Then $f\in \mathcal{F}^*$
\label{thm:sufficient_conditions_apd}.
\end{theorem}

\begin{proof}
For clarity, we divide this theorem into three sub-theorems and present their proofs in the following subsections.
\end{proof}

\subsubsection{Proof of the Theorem~\ref{thm:sufficient_conditions_apd}.1}

Theorem~\ref{thm:sufficient_conditions_apd}.1 demonstrates that:

If 
\begin{itemize}
    \item \( f = h \circ g \) is an optimal hypothesis for all training domains, i.e., $f \in \bigcap_{{e} \in \mathcal{E}_{tr}} \underset{f \in \mathcal{F}}{\text{argmin}} \ \mathcal{L}(f, \mathbb{P}^{e}),$
    \item and \( g \) is an invariant representation function, i.e., \( g \in \mathcal{G}_c \),
\end{itemize}
then \( f \in \mathcal{F}^* \).

To prove this, in the following theorem, we show that for any domain \( \mathbb{P}^e \) satisfying causal support (Assumption~\ref{as:sufficient_causal_support}), if \( f = h \circ g \) is optimal for \( \mathbb{P}^e \) and \( g \in \mathcal{G}_c \), then \( f \in \mathcal{F}^* \). Note that in the following theorem, for simplicity, we assume that
$\mathbb{P}^e$ is a mixture of the training domains. Therefore, $\mathcal{E}_{tr}$ satisfying causal support implies \( \mathbb{P}^e \) also satisfying causal support i.e., $\text{supp}\{\mathbb{P}^e(Z_c)\}=\mathcal{Z}_c$.

\begin{theorem} Denote the set of \textbf{domain optimal hypotheses} of $\mathbb{P}^e$ induced by $g\in \mathcal{G}$: 
    \begin{equation*}
        \mathcal{F}_{\mathbb{P}^e,g}=\left \{h\circ g \mid h\in\underset{h'\in \mathcal{H}}{\rm{argmin }} \mathcal{L}\left ( h'\circ g, {\mathbb{P}^{e}} \right )  \right \}.
    \end{equation*} 
If $\text{supp}\{\mathbb{P}^e(Z_c)\}=\mathcal{Z}_c$ and $g\in \mathcal{G}_c$, then $\mathcal{F}_{\mathbb{P}^e,g}  \subseteq \mathcal{F}^{*}$. 
\label{thm:single_generalization_apd}
\end{theorem}

\begin{proof}



Given $\text{supp}\{\mathbb{P}^e(Z_c)\}=\mathcal{Z}_c$ and $g_c\in \mathcal{G}_c$, 
it suffices to prove that for any $f_c = h_c \circ g_c \in \mathcal{F}_{\mathbb{P}^e,g_c}$, we have:

\begin{equation}
    f_c \in \bigcap_{\mathbb{P}^{e}\in \mathcal{P}} \underset{f\in \mathcal{F}}{\text{argmin}}\mathcal{L}\left(f, \mathbb{P}^e\right).    
    \label{eq:single_generalization_apd}
\end{equation}

To prove (\ref{eq:single_generalization_apd}), we only need to show that for any $f=h\circ g_c \in\mathcal{F}$ and $\mathbb{P}^{e'} \in \mathcal{P}$:

\begin{equation}
\mathcal{L}\left(f, \mathbb{P}^{e'}\right)
\geq \mathcal{L}\left(f_c, \mathbb{P}^{e'}\right),
\end{equation}

which is equivalent to:
\begin{equation}
\mathbb{E}_{(x,y)\sim\mathbb{P}^{e'}}\left[\ell\left(f\left(x\right),y\right)\right]
\geq \mathbb{E}_{(x,y)\sim\mathbb{P}^{e'}}\left[\ell\left(f_c\left(x\right),y\right)\right].
\label{eq:target_apd}
\end{equation}


\underline{\textit{Step 1: Simplifying the general loss using the invariant representation function \(g_c\).}} 

Based on structural causal model (SCM) depicted in Figure~\ref{fig:graph} we have a distribution (domain) over the observed variables $(X,Y)$ given the environment $E=e \in \mathcal{E}$: \begin{align*}  \mathbb{P}^e(X,Y)&=\int_{\mathcal{Z}_c}\int_{\mathcal{Z}_e}\mathbb{P}^{e}(X, Y, Z_c=z_c,Z_e=z_e)d_{z_c} d_{z_e}\\
&=\int_{\mathcal{Z}_c}\int_{\mathcal{Z}_e}\mathbb{P}^{e}(X, Y, z_c,z_e)d_{z_c} d_{z_e}\\
   &= \int_{\mathcal{Z}_c}\int_{\mathcal{Z}_e}\mathbb{P}^{e}(X\mid z_c, z_e)\mathbb{P}^{e}(Y\mid z_c)\mathbb{P}^{e}(z_c,z_e) d_{z_c} d_{z_e}\\
   &= \int_{\mathcal{Z}_c}\int_{\mathcal{Z}_e}\mathbb{P}^{e}(z_c,z_e)\int_{\mathcal{X}}\mathbb{P}^{e}(X=x\mid z_c, z_e)\mathbb{P}^{e}(Y\mid z_c)d_{z_c} d_{z_e} d_x\\
   &= \int_{\mathcal{Z}_c}\int_{\mathcal{Z}_e}\mathbb{P}^{e}(z_c,z_e)\int_{\mathcal{X}}\mathbb{P}^{e}(X=x\mid z_c, z_e)\int_{\mathcal{Y}}\mathbb{P}^{e}(Y=y\mid z_c) d_{z_c} d_{z_e} d_x d_y\\
    &= \int_{\mathcal{Z}_c}\int_{\mathcal{Z}_e}\mathbb{P}^{e}(z_c,z_e)\int_{\mathcal{X}}\int_{\mathcal{U}_x}\mathbb{P}^{e}(X=x\mid z_c, z_e,u_x)\mathbb{P}^{e}(u_x)\int_{\mathcal{Y}}\mathbb{P}^{e}(Y=y\mid z_c)  d_{z_c} d_{z_e} d_x d_y d_{u_x}\\
    &\stackrel{(1)}{=} \int_{\mathcal{Z}_c}\int_{\mathcal{Z}_e}\mathbb{P}^{e}(z_c,z_e)\int_{\mathcal{X}}\int_{\mathcal{U}_x}\mathbb{I}_{x= \psi_x(z_c, z_e,u_x)}\mathbb{P}^{e}(u_x)\int_{\mathcal{Y}}\mathbb{P}^{e}(Y=y\mid z_c)  d_{z_c} d_{z_e} d_x d_y d_{u_x}\\
\end{align*}

We have $\stackrel{(1)}{=}$ by definition of SCM, $x$ is the deterministic function of $(z_c, z_e,u_x)$.

Therefore we have:

\begin{align}
&\mathbb{E}_{(x,y)\sim\mathbb{P}^{e}(X,Y)}\left[\ell\left(f\left(x\right),y\right)\right]\nonumber\\
&= 
\int_{\mathcal{Z}_c}\int_{\mathcal{Z}_e}\mathbb{P}^{e}(z_c,z_e)\int_{\mathcal{X}}\int_{\mathcal{U}_x}\mathbb{I}_{x= \psi_x(z_c, z_e,u_x)}\mathbb{P}^{e}(u_x)\int_{\mathcal{Y}}\mathbb{P}^{e}(Y=y\mid z_c)\ell\left(f\left(x\right),y\right)  d_{z_c} d_{z_e} d_x d_y d_{u_x}\nonumber\\
&= 
\int_{\mathcal{Z}_c}\int_{\mathcal{Z}_e}\mathbb{P}^{e}(z_c,z_e)\int_{\mathcal{U}_x}\int_{\mathcal{Y}}\mathbb{P}^{e}(Y=y\mid z_c)\int_{\mathcal{X}}\mathbb{I}_{x= \psi_x(z_c, z_e,u_x)}\ell\left(f\left(x\right),y\right) \mathbb{P}^{e}(u_x) d_{z_c} d_{z_e} d_x d_y d_{u_x}\nonumber\\
&= 
\int_{\mathcal{Z}_c}\int_{\mathcal{Z}_e}\mathbb{P}^{e}(z_c,z_e)\int_{\mathcal{U}_x}\int_{\mathcal{Y}}\mathbb{P}^{e}(Y=y\mid z_c)\int_{\mathcal{X}}\mathbb{I}_{x= \psi_x(z_c, z_e,u_x)}\ell\left(f\left(\psi_x(z_c, z_e,u_x)\right),y\right) \mathbb{P}^{e}(u_x) d_{z_c} d_{z_e} d_x d_y d_{u_x}\nonumber
\\
&= 
\int_{\mathcal{Z}_c}\int_{\mathcal{Z}_e}\mathbb{P}^{e}(z_c,z_e)\int_{\mathcal{U}_x}\int_{\mathcal{Y}}\mathbb{P}^{e}(Y=y\mid z_c)\ell\left(f\left(\psi_x(z_c, z_e,u_x)\right),y\right) \mathbb{P}^{e}(u_x) d_{z_c} d_{z_e} d_y d_{u_x}\nonumber
\\
&= 
\int_{\mathcal{Z}_c}\int_{\mathcal{Z}_e}\mathbb{P}^{e}(z_c,z_e)\int_{\mathcal{U}_x}\mathbb{E}_{y\sim\mathbb{P}(Y\mid z_c)} \left[ \ell\left(f\left(\psi_x(z_c, z_e,u_x)\right),y\right)\right]
 \mathbb{P}^{e}(u_x) d_{z_c} d_{z_e}  d_{u_x}\nonumber
\\
&= 
\int_{\mathcal{Z}_c}\int_{\mathcal{Z}_e}\mathbb{P}^{e}(z_c,z_e)\int_{\mathcal{U}_x}\mathbb{E}_{y\sim\mathbb{P}(Y\mid z_c)} \left[ \ell\left((h\circ g_c)\left(\psi_x(z_c, z_e,u_x)\right),y\right)\right]
 \mathbb{P}^{e}(u_x) d_{z_c} d_{z_e}  d_{u_x}\nonumber
\\
&\stackrel{(1)}{=} 
\int_{\mathcal{Z}_c}\int_{\mathcal{Z}_e}\mathbb{P}^{e}(z_c,z_e)\int_{\mathcal{U}_x}\mathbb{E}_{y\sim\mathbb{P}(Y\mid z_c)} \left[ \ell\left(h\left(T(z_c)\right),y\right)\right]
 \mathbb{P}^{e}(u_x) d_{z_c} d_{z_e}  d_{u_x}\nonumber\\
&=
\int_{\mathcal{Z}_c}\mathbb{P}^{e}(z_c)\mathbb{E}_{y\sim\mathbb{P}(Y\mid z_c)} \left[ \ell\left(h\left(T(z_c)\right),y\right)\right]
 d_{z_c} \nonumber\\
&= 
\int_{\mathcal{Z}_c}\mathbb{P}^{e}(z_c)\mathbb{E}_{y\sim\mathbb{P}(Y\mid T(z_c))} \left[ \ell\left(h\left(T(z_c)\right),y\right)\right]
 d_{z_c} \nonumber
\\
&\stackrel{(2)}{=} 
\int_{\mathcal{Z}_c}T_{\#}\mathbb{P}^{e}(z_c)\mathbb{E}_{y\sim\mathbb{P}(Y\mid z_c)} \left[ \ell\left(h\left(z_c\right),y\right)\right]
 d_{z_c} \nonumber
\end{align}

We have:
\begin{itemize}
    \item $\stackrel{(1)}{=}$ by property-2 of $g_c$ (Corollary~\ref{cor:proterties});
    \item $\stackrel{(2)}{=}$ because $T: \mathcal{Z}_c\rightarrow \mathcal{Z}_c$ and  $T_{\#}\mathbb{P}^{e}(z_c)=\int_{z^{'}_c\in T^{-1}(z_c)}\mathbb{P}^{e}(z^{'}_c)d_{z^{'}_c}$ 
\end{itemize} 

Now, to prove (\ref{eq:target_apd}), we only need to show:


\begin{align}
\int_{\mathcal{Z}_c}T_{\#}\mathbb{P}^{e'}(z_c)\mathbb{E}_{y\sim\mathbb{P}(Y\mid z_c)} \left[ \ell\left(h_c\left(z_c\right),y\right)\right]
 d_{z_c}\leq\int_{\mathcal{Z}_c}T_{\#}\mathbb{P}^{e'}(z_c)\mathbb{E}_{y\sim\mathbb{P}(Y\mid z_c)} \left[ \ell\left(h\left(z_c\right),y\right)\right]
 d_{z_c}
 \label{eq:target_causal}
\end{align}




\underline{\textit{Step 2: Generalization of $h_c$.}} \textit{Step-1} Demonstrate that \(h_c\) only needs to make predictions for the set of causal factors \(z_c \in \mathcal{Z}_c\). Therefore, it is sufficient to show that \(h_c\) is optimal for every \(z \in \mathcal{Z}_c\).

Recall that $f_c=h_c\circ g_c\in \mathcal{F}_{\mathbb{P}^e,g_c}$, 
therefore, 
$$h_c\in \underset{h\in \mathcal{H}}{\text{argmin }} \int_{\mathcal{Z}_c}T_{\#}\mathbb{P}^{e}(z_c)\mathbb{E}_{y\sim\mathbb{P}(Y\mid z_c)} \left[ \ell\left(h\left(z_c\right),y\right)\right]
 d_{z_c} $$

By property-3 of $g_c$ (Corollary~\ref{cor:proterties}), there exists an optimal function \(h^*\) such that:
\begin{equation*}
  h^* \in\bigcap_{z_c\in\mathcal{Z}_c} \underset{h\in \mathcal{H}}{\text{argmin }} \mathbb{E}_{y\sim\mathbb{P}(Y\mid z_c)} \ell\left ( h( z_c), y \right ),  
\end{equation*}


Property-3 of \(g_c\) ensures the existence of an optimal \(h^*\) for every causal factor \(z_c \in \mathcal{Z}_c\), it follows that \(h_c\) must also be optimal for every causal feature \(z_c\) within its support, \(\text{supp}\,\mathbb{P}^e(Z_e)\). This implies that \(h_c(z_c) = h^*(z_c)\) for every \(z_c\) where \(\mathbb{P}^e(z_e) > 0\).

Moreover, since \(\text{supp}\,\mathbb{P}^e(Z_e) = \mathcal{Z}_c\), this implies that \(h_c(z_c) = h^*(z_c)\) for every \(z_c \in \mathcal{Z}_c\).

\underline{\textit{Step-3: Proof of (\ref{eq:target_causal})}}.

\begin{align*}
\int_{\mathcal{Z}_c}T_{\#}\mathbb{P}^{e'}(z_c)\mathbb{E}_{y\sim\mathbb{P}(Y\mid z_c)} \left[ \ell\left(h_c\left(z_c\right),y\right)\right]
 d_{z_c}\leq\int_{\mathcal{Z}_c}T_{\#}\mathbb{P}^{e'}(z_c)\mathbb{E}_{y\sim\mathbb{P}(Y\mid z_c)} \left[ \ell\left(h\left(z_c\right),y\right)\right]
 d_{z_c}
\end{align*}

From \textit{step-2}, we have 

$$\mathbb{E}_{y\sim\mathbb{P}(Y\mid z_c)} \left[ \ell\left(h_c\left(z_c\right),y\right)\right]
\leq\mathbb{E}_{y\sim\mathbb{P}(Y\mid z_c)} \left[ \ell\left(h\left(z_c\right),y\right)\right]
$$
for all $z_c\in \mathcal{Z}_c$. By taking the expectation and applying the law of iterated expectation, inequality (\ref{eq:target_causal}) follows. This concludes the proof.

\end{proof}

\subsubsection{Proof of the Theorem~\ref{thm:sufficient_conditions_apd}.2}

Similar to the proof of Theorem~\ref{thm:sufficient_conditions_apd}.1, in the following result, we assume for simplicity that \( \mathbb{P}^e \) is a mixture of the training domains.
Then, Theorem~\ref{thm:sufficient_conditions_apd}.2 is stated as follows:

\begin{theorem}Under Assumption \ref{as:label_idf} and Assumption \ref{as:sufficient_causal_support}, 
if \( f\) is an optimal hypothesis for $\mathbb{P}^{e}$ i.e., $ f \in \underset{f \in \mathcal{F}}{\text{argmin}} \ \mathcal{L}(f, \mathbb{P}^{e}),$ and the support of joint causal and spurious factors of $\mathbb{P}^{e}$ covers the entire causal and spurious factor space $\mathcal{Z}_c\times\mathcal{Z}_e$ i.e., $\text{supp}\{\mathbb{P}^{e} \left (Z_c, Z_e \right )\}=\mathcal{Z}_c\times\mathcal{Z}_e$, then then \( f \in \mathcal{F}^* \).
\end{theorem}

\begin{proof}

Based on structural causal model (SCM) depicted in Figure~\ref{fig:graph} we have a distribution (domain) over the observed variables $(X,Y)$ given the environment $E=e \in \mathcal{E}$: 

Therefore we have:

\begin{align}
&\mathbb{E}_{(x,y)\sim\mathbb{P}^{e}(X,Y)}\left[\ell\left(f\left(x\right),y\right)\right]\nonumber\\
&= 
\int_{\mathcal{Z}_c}\int_{\mathcal{Z}_e}\mathbb{P}^{e}(z_c,z_e)\int_{\mathcal{X}}\int_{\mathcal{U}_x}\mathbb{I}_{x= \psi_x(z_c, z_e,u_x)}\mathbb{P}^{e}(u_x)\int_{\mathcal{Y}}\mathbb{P}^{e}(Y=y\mid z_c)\ell\left(f\left(x\right),y\right)  d_{z_c} d_{z_e} d_x d_y d_{u_x}\nonumber\\
&= 
\int_{\mathcal{Z}_c}\int_{\mathcal{Z}_e}\mathbb{P}^{e}(z_c,z_e)\int_{\mathcal{U}_x}\int_{\mathcal{Y}}\mathbb{P}^{e}(Y=y\mid z_c)\int_{\mathcal{X}}\mathbb{I}_{x= \psi_x(z_c, z_e,u_x)}\ell\left(f\left(x\right),y\right) \mathbb{P}^{e}(u_x) d_{z_c} d_{z_e} d_x d_y d_{u_x}\nonumber\\
&= 
\int_{\mathcal{Z}_c}\int_{\mathcal{Z}_e}\mathbb{P}^{e}(z_c,z_e)\int_{\mathcal{U}_x}\int_{\mathcal{Y}}\mathbb{P}^{e}(Y=y\mid z_c)\int_{\mathcal{X}}\mathbb{I}_{x= \psi_x(z_c, z_e,u_x)}\ell\left(f\left(\psi_x(z_c, z_e,u_x)\right),y\right) \mathbb{P}^{e}(u_x) d_{z_c} d_{z_e} d_x d_y d_{u_x}\nonumber
\\
&= 
\int_{\mathcal{Z}_c}\int_{\mathcal{Z}_e}\mathbb{P}^{e}(z_c,z_e)\int_{\mathcal{U}_x}\int_{\mathcal{Y}}\mathbb{P}^{e}(Y=y\mid z_c)\ell\left(f\left(\psi_x(z_c, z_e,u_x)\right),y\right) \mathbb{P}^{e}(u_x) d_{z_c} d_{z_e} d_y d_{u_x}\nonumber
\\
&= 
\int_{\mathcal{Z}_c}\int_{\mathcal{Z}_e}\mathbb{P}^{e}(z_c,z_e)\int_{\mathcal{U}_x}\mathbb{E}_{y\sim\mathbb{P}(Y\mid z_c)} \left[ \ell\left(f\left(\psi_x(z_c, z_e,u_x)\right),y\right)\right]
 \mathbb{P}^{e}(u_x) d_{z_c} d_{z_e}  d_{u_x}\nonumber
\\
&= 
\int_{\mathcal{Z}_c}\int_{\mathcal{Z}_e}\mathbb{P}^{e}(z_c,z_e)\int_{\mathcal{U}_x}\mathbb{E}_{y\sim\mathbb{P}(Y\mid z_c)} \left[ \ell\left((h\circ g_c)\left(\psi_x(z_c, z_e,u_x)\right),y\right)\right]
 \mathbb{P}^{e}(u_x) d_{z_c} d_{z_e}  d_{u_x}
\end{align}
\end{proof}

Under Assumption \ref{as:label_idf}, 
given \( f\) is an optimal hypothesis for $\mathbb{P}^{e}$ i.e., $ f \in \underset{f \in \mathcal{F}}{\text{argmin}} \ \mathcal{L}(f, \mathbb{P}^{e}),$ then 
\begin{align*}
f\in\int_{\mathcal{U}_x}\mathbb{E}_{y\sim\mathbb{P}(Y\mid z_c)} \left[ \ell\left((h\circ g_c)\left(\psi_x(z_c, z_e,u_x)\right),y\right)\right]
 \mathbb{P}^{e}(u_x)  d_{u_x}
\end{align*}

This holds because, under Assumption~\ref{as:label_idf}, it is guaranteed that there exists an optimal prediction for every $x=\psi_x(z_c, z_e, u_x)$ for all $\{(z_c,z_e)\sim \mathbb{P}^{e}(z_c,z_e),u_x\sim \mathbb{P}^{e}(u_x)\}$.

Furthermore, since the support of the joint causal and spurious factors in \( \mathbb{P}^{e} \) spans the entire causal and spurious factor space, i.e., $\text{supp}\{\mathbb{P}^{e} (Z_c, Z_e)\} = \mathcal{Z}_c \times \mathcal{Z}_e,$ the hypothesis \( f \) remains optimal for all possible configurations of \( x = \psi_x(z_c, z_e, u_x) \) across all \( z_c, z_e, u_x \). This implies that \( f \in \mathcal{F}^* \).

\textbf{Note:} 
\begin{itemize}
    \item It is important to highlight that this theorem aligns with Theorem 3 from \citep{ahuja2021invariance}. However, their analysis is conducted in a linear setting.
    \item We argue that the sub-condition of "sufficient and diverse training domains" is impractical, making it a weak guarantee for ensuring the generalization of algorithms based on this condition.
\end{itemize}

\subsubsection{Proof of the Theorem~\ref{thm:sufficient_conditions_apd}.3}

Similar to the proof of Theorem~\ref{thm:sufficient_conditions_apd}.1, in the following result, we assume for simplicity that \( \mathbb{P}^e \) is a mixture of the training domains.
Then, Theorem~\ref{thm:sufficient_conditions_apd}.3 is stated as follows:

\begin{theorem}Under Assumption \ref{as:label_idf} and Assumption \ref{as:sufficient_causal_support}, given $\mathcal{T}$ is set of all \textbf{invariance-preserving transformations} such that for any $T\in \mathcal{T}$ and $g_c\in \mathcal{G}_c$: $(g_c\circ T)(\cdot)=g_c(\cdot)$,
if \( f\) is an optimal hypothesis for $\mathbb{P}^{e}$ i.e., $ f \in \underset{f \in \mathcal{F}}{\text{argmin}} \ \mathcal{L}(f, \mathbb{P}^{e}),$ and  $f$ is also an optimal hypothesis on all augmented domains i.e., $$f\in \bigcap_{T\in \mathcal{T}}\underset{f\in \mathcal{F}}{\text{argmin}} \ \loss{f,T\#\mathbb{P}^{e}}$$
 
, then then \( f \in \mathcal{F}^* \).
\end{theorem}

\begin{proof}
We first analyze the characteristics of the set of all invariance-preserving transformations \( \mathcal{T} \).

By the definition of \( \mathcal{T} \) and the set of invariant representations \( \mathcal{G}_c \):

\begin{itemize}
    \item given $T\in \mathcal{T}$ and $g_c\in \mathcal{G}_c$, we have $(g_c\circ T)(x)=g_c(x)$ for all $x=\psi(z_c,z_e,u_x)$ (for all $z_c\in \mathcal{Z}_c, z_c\in \mathcal{Z}_c$, $u_x\in \mathcal{U}_x$).
    \item  $g\in \mathcal{G}_c$, $\mathbb{P}(Y\mid g(x)) = \mathbb{P}(Y\mid z_c)$ and $g(x)=g(x')$ holds true for all $\{(x,x',z_c)\mid  x= \psi_x(z_c, z_e, u_x), x'= \psi_x(z_c, z^{'}_e, u^{'}_x) \text{ for all }z_e,z^{'}_e, u_x, u^{'}_x\}$
\end{itemize}

This implies that for any $ x=\psi_x(z_c, z_e, u_x)$, we have:
\begin{equation*}
    T(x) \in \left \{\psi_x(z'_c, z'_e, u'_x) \text{ where } \left (z'_c\in\mathcal{Z}_c,z'_e\in \mathcal{Z}_e, u'_x\in \mathcal{U}_x \text{ and } P(Y\mid Z_c=z_c)=P(Y\mid Z_c=z'_c)\right ) \right \}
\end{equation*}

In other words, given a sample $\psi_x(z_c, z_e, u_x)$, the transformation $T$ operates as follows:
\begin{itemize}
    \item It augments \( z_c \) to its equivalent \( z'_c \), ensuring that $P(Y\mid Z_c=z_c)=P(Y\mid Z_c=z'_c)$
    \item It modifies the environmental (or spurious) feature $z_e$ to any $z'_e \in \mathcal{Z}_e$.
    \item It applies changes to the noise term $u_x$.
\end{itemize}

Therefore, under Assumption~\ref{as:sufficient_causal_support} (causal support), having access to all \( T \in \mathcal{T} \) is equivalent to having sufficient and diverse training domains. That is, the support of the joint causal and spurious factors in \( \mathbb{P}^{e} \) spans the entire causal and spurious factor space i.e., $\text{supp}\{\mathbb{P}^{e} (Z_c, Z_e)\} = \mathcal{Z}_c \times \mathcal{Z}_e$. This concludes the proof.

\textbf{Note:} In general, accessing all transformations from $\mathcal{T}$ is impractical. However, recently, some works leverage foundation models to generate these augmentations, achieving promising empirical performance \cite{ruan2021optimal}.
\end{proof}

\subsection{Necessary Conditions for achieving Generalization}



\begin{theorem} \textbf{(Theorem \ref{thm:nacessary} in the main paper)} Given representation function $g$,
$\exists h: h\circ g\in \mathcal{F}^*$ if and only if $g\in \mathcal{G}_s$.
\label{thm:nacessary_apd}
\end{theorem}

\begin{proof} \textit{\textbf{``if"} direction.} If $g\in \mathcal{G}_s$, we have: 

\begin{enumerate}
    \item By definition of $g\in \mathcal{G}_s$ we have
$I(g(X),g_c(X))=I(X,g_c(X)$ i.e., $g(X)$ retain all information about $g_c(X)$ presented in $X$. Therefore, there exists a function \(\phi\) such that \(\phi \circ g \in \mathcal{G}_c\), which implies the existence of a \(g_c\in \mathcal{G}_c\) such that \(\phi \circ g = g_c\).

    \item By the definition of \(g_c \in \mathcal{G}_c\), we can always find a classifier \(h_c\) such that \(h_c \circ g_c \in \mathcal{F}^*\).

\end{enumerate}

To complete the proof of the \textbf{``if"} direction, we need to demonstrate the existence of a classifier \( h \) on top of the representation induced by \( g \) such that it forms a globally optimal hypothesis, i.e., \( h \circ g \in \mathcal{F}^* \).

From (1) and (2) we have $h_c\circ g_c = h_c\circ \phi \circ g \in \mathcal{F}^*$. Therefore, we can construct classifier $h = h_c \circ \phi$, then $h \circ g =  h_c\circ \phi \circ g =h_c\circ g_c \in \mathcal{F}^*$.
\end{proof}

\begin{proof} \textit{\textbf{``only if"} direction by contraction.} 

Define the set of optimal hypotheses induced by a representation function \( g \) as:

\begin{equation*}
\mathcal{F}_{g,\mathcal{E}_{tr}}=\left\{ h\circ g: \bigcap_{{e} \in \mathcal{E}_{tr}} \underset{h\circ g \in \mathcal{F}}{\text{argmin}} \ \mathcal{L}(h\circ g, \mathbb{P}^{e})\right \}
\end{equation*}  

We show that if $g$ is not sufficient-representation, for any $f\in\mathcal{F}_{g,\mathcal{E}_{tr}}$ there exists multiple target domains where $f$ performs arbitrarily bad.

By definition of $g\notin \mathcal{G}_s$, we have
$I(g(X),g_c(X))<I(X,g_c(X)$. Therefore, there does not exist a function $\phi$ such that $\phi\circ g \in \mathcal{G}_c$. This implies we can not construct any classifier $h = h_c \circ \phi$, then $h \circ g =  h_c\circ \phi \circ g =h_c\circ g_c \in \mathcal{F}^*$.

Consequently, $h$ has to rely on spurious feature $z_e$ (or both $z_c$ and $z_e$) to make predict for some $\{x\mid x=\psi_x\{z_c,z_e, u_x\} \text{ for some } z_c \text{ such that } \mathbb{P}(Y\mid Z_e=z_e)= \mathbb{P}(Y\mid Z_c=z_c) \}$.  

Note that based on structural causal model (SCM) depicted in Figure~\ref{fig:graph}, we have $Z_e\not\!\perp\!\!\!\perp Y$ i.e., the environmental feature $Z_e$ spuriously correlated with $Y$. Therefore, there is a set $\mathcal{B}=\{x\mid x=\psi_x\{z^{'}_c, z_e, u_x\} \text{ for some } z^{'}_c \text{ such that } \mathbb{P}(Y\mid Z_e = z_e)\neq \mathbb{P}(Y\mid Z_c=z^{'}_c)\} \neq \emptyset$. Consequently, $h(\phi(g(x))) \neq h_c(g_c(x))$ for all $x\in \mathcal{B}$

    We can construct undesirable target domains $\mathbb{P}^{e_i}$ with arbitrary loss $\mathcal{L}(h\circ g, \mathbb{P}^{e_i})$ by giving $(1-\delta)$ percentage mass to that examples in $\mathcal{B}$ and $(\delta)$ percentage mass that examples in $\mathcal{X} \setminus \mathcal{B}$. This is equivalent to 
\begin{equation}
    \mathbb{E}_{(x,y)\sim\mathbb{P}^{e_i}} \left [ h(g(x)) \neq h_c(g_c(x))  \right ]  \geq 1-\delta.\nonumber
\end{equation}
 with $(0\leq\delta\leq 1)$.

This concludes the proof.

\end{proof}

\begin{corollary} \textbf{(Corollary~\ref{thm:information} in the main paper)} Under Assumption \ref{as:label_idf} and Assumption \ref{as:sufficient_causal_support}, let the minimal representation function $g_{\text{min}}$ be defined as:
\begin{equation}
g_{\text{min}} \in \mathcal{G}_{min}=\left\{\underset{g \in \mathcal{G}}{\text{argmin }} I(g(X); X) \ \text{s.t.} \ f = h \circ g \in \mathcal{F}_{\mathcal{E}_{\text{tr}}} \right\},
\label{eq:minimal}
\end{equation}
where $I$ denotes mutual information. Then, for any $g_c\in \mathcal{G}_c$ the following holds:
\begin{equation}
I(g_{\text{min}}(X), g_c(X)) \leq I(X, g_c(X)),
\end{equation}
and the equality holds if and only if at least one of sufficient conditions is hold.
\label{thm:information_apd}.
\end{corollary}

\begin{proof} We first prove that if one of the sufficient conditions holds, then the following equality holds:
\[
I(g_{\text{min}}(X), g_c(X)) = I(X, g_c(X)).
\]

Define:
\begin{equation*}
\mathcal{G}_{\mathcal{E}_{tr}}=\left \{g: f = h \circ g \in \mathcal{F}_{\mathcal{E}_{\text{tr}}} \right\}.
\end{equation*}
By Theorem~\ref{thm:sufficient_conditions_apd}, if one of the sufficient conditions holds, then \( \mathcal{G}_{\mathcal{E}_{tr}} \subseteq \mathcal{G}_c \).

From the definition in Eq.~(\ref{eq:minimal}), we have:
\[
\mathcal{G}_{\text{min}} \subseteq \mathcal{G}_{\mathcal{E}_{tr}} \subseteq \mathcal{G}_c.
\]
This implies:
\[
I(g_{\text{min}}(X), g_c(X)) = I(g_c(X), g_c(X)) = I(X, g_c(X)).
\]
\end{proof}

\begin{proof} We prove that if the equality holds, then \( g_{\text{min}} \in \mathcal{G}_c \).

If
\[
I(g_{\text{min}}(X), g_c(X)) = I(X, g_c(X)),
\]
then it follows that
\[
I(g_{\text{min}}(X), X) \geq I(X, g_c(X)).
\]
Therefore, by the definition of \( g_{\text{min}} \), we conclude that \( g_{\text{min}} \in \mathcal{G}_c \).

\end{proof}

\subsection{Representation Alignment trade-off}
\label{apd:tradeoff}
As a reminder, $\mathbb{P}$ denotes data distribution on data space $\mathcal{X}$, while $g_{\#}\mathbb{P}$ denotes latent distribution on full latent space $\mathcal{Z}$, with $g: \mathcal{X} \mapsto \mathcal{Z}$ is the encoder. 

In the following, we recap the theoretical results for Hellinger distance as presented by \cite{phung2021learning}. Similar results for $\mathcal{H}$-divergence can be found in Zhao et al. \cite{zhao2019learning}, and for Wasserstein distance in Le et al. \cite{le2021lamda}.

\subsubsection{Upper Bound}

\begin{theorem} 
\label{thm:single_bound_A}Consider the source domain
$\mathbb{P}^{e'}$ and the
target domain $\mathbb{P}^{e}$. Let $\ell$ be any loss function
upper-bounded by a positive constant $L$. For any hypothesis $f:\mathcal{X}\mapsto\mathcal{Y}_{\Delta}$
where $f=h\circ g$ with $g:\mathcal{X}\mapsto\mathcal{Z}$
and $h:\mathcal{Z}\mapsto\mathcal{Y}_{\Delta}$, the target
loss on input space is upper bounded 
\begin{equation}
\begin{aligned}\mathcal{L}\left(f,\mathbb{P}^{e}\right)\leq\mathcal{L}\left(f,\mathbb{P}^{e'}\right)+L\sqrt{2}\,d_{1/2}\left(\mathbb{P}_{g}^{e},\mathbb{P}_{g}^{e'}\right)\end{aligned}
,\label{eq:input_bound_1-1}
\end{equation}
\end{theorem}

This Theorem is directly adapted from the result of Trung et al. \cite{phung2021learning}.
The upper bound for target loss above relates source loss, target loss and data shift on feature space, which is different to other bounds in which the data shift is on input space.

\subsubsection{Lower Bound}
\begin{theorem}
\label{theorem:single_lower_bound_A}
\cite{phung2021learning} Consider a hypothesis $f=h\circ g$, the Hellinger distance between two label marginal distributions $\mathbb{P}^{e'}$ and $\mathbb{P}^{e}$ can be upper-bounded as: 
\begin{equation}
d_{1/2}\left(\mathbb{P}^{e'}_\mathcal{Y},\mathbb{P}^{e}_\mathcal{Y}\right) \leq 
\mathcal{L}\left ( f,\mathbb{P}^{e'} \right )^{1/2}+
d_{1/2}\left ( g_{\#}\mathbb{P}^{e'},g_{\#}\mathbb{P}^{e} \right )+
\mathcal{L}\left ( f,\mathbb{P}^{e} \right )^{1/2}
\end{equation}

where the general loss $\mathcal{L}$ is defined based on the Hellinger loss $\ell$ which is define as $\ell\left ( f(x) \right )=D_{1/2}\left ( f(x),\mathbb{P}(Y\mid x) \right )=2\sum_{i=1}^C\left ( \sqrt{f(x,i)}-\sqrt{\mathbb{P}(Y=i\mid x)} \right )^2$.
\end{theorem}

\subsection{Subspace Representation Alignment}

In the following, we prove the theoretical results for Hellinger distance based on the findings of Trung et al. \cite{phung2021learning}. A similar strategy can be directly applied to $\mathcal{H}$-divergence \cite{zhao2019learning} and Wasserstein distance \cite{le2021lamda}.

\begin{theorem}
\label{theorem:multi_bound_A} \textbf{(Theorem~\ref{theorem:multi_bound})} 
\textit{(Subspace Representation Alignment)} Given \textit{subspace projector} $\Gamma: \mathcal{X}\rightarrow \mathcal{M}$, and a subspace index $m\in \mathcal{M}$, let $A_{m}=\Gamma^{-1}(m)=\left\{ x:\Gamma(x)=m\right\} $ is the region on data space which has the same index $m$ and $\mathbb{P}_{m}^{e}$ be the distribution restricted by $\mathbb{P}^{e}$ over the set $A_{m}$, then $\pi^{e}_m=\frac{\mathbb{P}^{e}\left(A_{m}\right)}{\sum_{m'\in\mathcal{M}}\mathbb{P}^{e}\left(A_{m'}\right)}$ is mixture co-efficient, if the loss function $\ell$ is upper-bounded by a positive
constant $L$, then:

(i)  The target general loss is upper-bounded: 
\begin{align*}
\left | \mathcal{E}_{tr} \right |\sum_{e\in \mathcal{E}_{tr}}\mathcal{L}\left ( f,\mathbb{P}^{e} \right )
\leq
\sum_{e\in \mathcal{E}_{tr}} \sum_{m\in\mathcal{M}}\pi^{e}_m
\mathcal{L}\left ( f,\mathbb{P}^{e}_{m} \right ) +
L\sum_{e, e'\in \mathcal{E}_{tr}}\sum_{m\in\mathcal{M}}\pi^{e}_{m}D\left ( g_{\#}\mathbb{P}^{e}_{m},g_{\#}\mathbb{P}^{e'}_{m} \right ),
\end{align*}
(ii) Distance between two label marginal distribution $\mathbb{P}^{e}_{m}(Y)$ and $\mathbb{P}^{e'}_{m}(Y)$ can be upper-bounded: 
\begin{equation*}
\begin{aligned}
D\left(\mathbb{P}^{e}_{\mathcal{Y},m},\mathbb{P}^{e'}_{\mathcal{Y},m}\right) \leq 
D\left ( g_{\#}\mathbb{P}^{e}_{m},g_{\#}\mathbb{P}^{e'}_{m} \right )
+\mathcal{L}\left ( f,\mathbb{P}^{e}_{m}\right )
+
\mathcal{L}\left ( f,\mathbb{P}^{e'}_{m} \right )
\end{aligned}
\end{equation*}
(iii) Construct the subspace projector $\Gamma$ as the optimal hypothesis for the training domains i,e., \(\Gamma \in \mathcal{F}_{\mathcal{E}_{tr}}\), which defines $\mathcal{M}=\{m=\hat{y}\mid \hat{y}=\Gamma(x), x\in\bigcup_{e\in\mathcal{E}_{tr}}\text{supp}\mathbb{P}^{e} \}\subseteq \mathcal{Y}_\Delta$, then
$
D\left(\mathbb{P}^{e}_{\mathcal{Y}, m},\mathbb{P}^{e'}_{\mathcal{Y},m}\right)=0$ for all \(m \in \mathcal{M}\).

where $g_{\#}\mathbb{P}$ denotes representation distribution on $\mathcal{Z}$ induce by applying $g$ with $g: \mathcal{X} \mapsto \mathcal{Z}$ on data distribution $\mathbb{P}$, $D$ can be $\mathcal{H}$-divergence, Hellinger or Wasserstein distance.
\end{theorem}

\begin{proof}
We consider \textit{sub-space projector} $\Gamma: \mathcal{X}\rightarrow \mathcal{M}$, given a sub-space index $m\in \mathcal{M}$, we denote $A_{m}=\Gamma^{-1}(m)=\left\{ x:\Gamma(x)=m\right\} $ is the region on data space which has the same index $m$.
Let $\mathbb{P}_{m}^{e}$ be the distribution restricted by $\mathbb{P}^{e}$ over the set $A_{m}$ and $\mathbb{P}_{m}^{e}$ as the distribution restricted by $\mathbb{P}^{e}$
over $A_{m}$. Eventually, we define $\mathbb{P}_{m}^{e}\left(y\mid x\right)$ as the probabilistic labeling
distribution on the sub-space $\left(A_{m},\mathbb{P}_{m}^{e}\right)$,
meaning that if $x\sim\mathbb{P}_{m}^{e}$, $\mathbb{P}_{m}^{e}\left(y\mid x\right)=\mathbb{P}_{e}\left(y\mid x\right)$.
Similarly, we define if $x\sim\mathbb{P}_{m}^{e'}$, $\mathbb{P}_{m}^{e'}\left(y\mid x\right)=\mathbb{P}^{e'}\left(y\mid x\right)$. Due to this construction, any data sampled from $\mathbb{P}_{m}^{e}$
or $\mathbb{P}_{m}^{e'}$ have the same index $m=\Gamma(x)$. 
Additionally, since each data point $x \in \mathcal{X}$ corresponds to only a single $\Gamma(x)$, the data space is partitioned into disjoint sets, i.e., $\mathcal{X} = \bigcup_{m=1}^{\mathcal{M}} A_{m}$, where $A_m \cap A_n = \emptyset, \forall m \neq n$. 
Consequently, the general loss of the target domain becomes:
\begin{equation}
\mathcal{L}\left(f,\mathbb{P}^{e}\right):=\sum_{m\in\mathcal{M}}\pi^{e}_m\mathcal{L}\left(f,\mathbb{P}_{m}^{e}\right),\label{eq:subspace_loss}
\end{equation}
where $\mathcal{M}$ is the set of all feasible sub-spaces indexing $m$ and  $\pi^{e}_m=\frac{\mathbb{P}^{e}\left(A_{m}\right)}{\sum_{m'\in\mathcal{M}}\mathbb{P}^{e}\left(A_{m'}\right)}$.

\end{proof}

\begin{proof}(i):

Using the same proof
for a single space in Theorem \ref{thm:single_bound_A}, we obtain:
\begin{equation} \mathcal{L}\left(f,\mathbb{P}^{e}_m\right)
\leq
\mathcal{L}\left(f_m,\mathcal{\mathbb{P}}_{m}^{e'}\right)
+ 
L\sqrt{2} d_{1/2}\left(g_{\#}\mathbb{P}^{e}_{m},g_{\#}\mathbb{P}_{m}^{e'}\right)
\end{equation}

Since $\mathcal{L}\left(f,\mathbb{P}^{e}\right):= \sum_{m}\pi^{e}_m \mathcal{L}\left(f,\mathbb{D}_{m}^{e}\right)$, taking weighted average over $m\in \mathcal{M}$, we reach (ii):
\begin{equation}
\mathcal{L}\left(f,\mathbb{P}^{e}\right)
\leq
\sum_{m}\pi^{e}_m
\mathcal{L}\left(f_m,\mathbb{P}_{m}^{e'}\right)
+ 
L\sqrt{2}\sum_{m}\pi^{e}_m d_{1/2}\left(g_{\#}\mathbb{P}^{e}_{m},g_{\#}\mathbb{P}_{m}^{e'}\right)
\end{equation}

By summing over the training domains on the left-hand side, we obtain:

\begin{align} \sum_{e\in\mathcal{E}_{tr}}\mathcal{L}\left(f_{\mathcal{M}},\mathbb{P}^{e}\right)
\leq&
\sum_{e\in\mathcal{E}_{tr}}\sum_{m}\pi^{e}_m
\mathcal{L}\left(f_m,\mathbb{P}_{m}^{e'}\right)
+ 
\sum_{e\in\mathcal{E}_{tr}}L\sqrt{2}\sum_{m}\pi^{e}_m d_{1/2}\left(g_{\#}\mathbb{P}^{e}_{m},g_{\#}\mathbb{P}_{m}^{e'}\right) \nonumber
\end{align}
Summing over the training domains on the left-hand side again:

\begin{align}
\sum_{e'\in\mathcal{E}_{tr}}\sum_{e\in\mathcal{E}_{tr}}\mathcal{L}\left(f_{\mathcal{M}},\mathbb{P}^{e}\right)
\leq&
\sum_{e'\in\mathcal{E}_{tr}}\sum_{e\in\mathcal{E}_{tr}}\sum_{m}\pi^{e}_m
\mathcal{L}\left(f_m,\mathbb{P}_{m}^{e'}\right)\nonumber\\
&+ 
\sum_{e'\in\mathcal{E}_{tr}}\sum_{e\in\mathcal{E}_{tr}}L\sqrt{2}\sum_{m}\pi^{e}_m d_{1/2}\left(g_{\#}\mathbb{P}^{e}_{m},g_{\#}\mathbb{P}_{m}^{e'}\right)\nonumber
\end{align}

Finally, we obtain:
\begin{align}
\left | \mathcal{E}_{tr} \right | \sum_{e\in \mathcal{E}_{tr}}\mathcal{L}\left(f,\mathbb{P}^{e}\right)
\leq&
\sum_{e,e'\in \mathcal{E}_{tr}}\sum_{m\in\mathcal{M}}\pi^{e}_m
\mathcal{L}\left(f,\mathcal{\mathbb{P}}_{m}^{e'}\right)
+
\sum_{e,e'\in \mathcal{E}_{tr}}L\sqrt{2}\sum_{m\in\mathcal{M}}\pi^{e}_{m} d_{1/2}\left(g_{\#}\mathbb{P}^{e}_{m},g_{\#}\mathbb{P}_{m}^{e'}\right)
\end{align}
\end{proof}

\begin{proof}(ii):

We obtain (ii) directly by applying the results from Theorem \ref{theorem:single_lower_bound_A} to each individual sub-space, denoted by the index $m$.
\end{proof}

\begin{proof}(iii):

Within training domains, we anticipate that $f\in \cap_{e\in \mathcal{E}_{tr}}\mathcal{F}_{\mathbb{P}^{e}}$ will predict the ground truth label $f(x)=f^*(x)$ where $f^*\in \mathcal{F}^*$.We can define a projector \(\Gamma = f\), which induces a set of subspace indices $\mathcal{M}=\{m=\hat{y}\mid \hat{y}=f(x), x\in\bigcup_{e\in\mathcal{E}_{tr}}\text{supp}\mathbb{P}^{e} \}\subseteq \Delta_{\left | \mathcal{Y}\right |}$. As a result, given subspace index $m\in\mathcal{M}$, $\forall i \in \mathcal{Y}, \mathbb{P}^{e}_{\mathcal{Y},m}(Y=i) = \mathbb{P}^{e'}_{\mathcal{Y},m}(Y=i) = \sum_{x \in f^{-1}(m)}\mathbb{P}(Y=i\mid x) = m[i]$. Consequently, \(D\left(\mathbb{P}^{e}_{\mathcal{Y},m}, \mathbb{P}^{e'}_{\mathcal{Y},m}\right) = 0\) for all \(m \in \mathcal{M}\), allowing us to jointly optimize both \textit{domain losses} and \textit{representation alignment}.
\end{proof}

\

\section{Practical Methodology}
\label{Sec:practical}
In this section, we present the practical objectives to achieve Eq. (\ref{eq:final_objective}):

\begin{align}
\min_{f=h\circ g} \underset{\text{Subspace Representation Alignment}}{\underbrace{\sum_{e,e'\in \mathcal{E}_{tr}}\sum_{m\in \mathcal{M}}D\left( g\#\mathbb{P}_m^{e},g\#\mathbb{P}_m^{e'}\right)}}\text{ s.t. } \underset{\text{Training domain optimal hypothesis}}{\underbrace{f=h\circ g\in \bigcap_{e\in \mathcal{E}_{tr}}\underset{ f}{\text{argmin }} \mathcal{L}\left(f,\mathbb{P}^{e}\right)}}
   \label{eq:final_objective_apd}
\end{align}

where $\mathcal{M}=\{\hat{y}\mid \hat{y}=f(x), x\in\bigcup_{e\in\mathcal{E}_{tr}}\text{supp}\mathbb{P}^{e} \}$ and $D$ can be $\mathcal{H}$-divergence, Hellinger distance, Wasserstein distance.

In the following, we consider the encoder $g$, classifier $h$, domain discriminator $h_d$ and set of $K$ empirical training domains $\mathbb{D}^{e_i}=\{x_{j}^{e_i},y_{j}^{e_i}\}_{j=1}^{N_{e_i}}\sim [\mathbb{P}^{e_i} ]^{N_{e_i}}$, $i=1...K$.

\subsection{Optimal hypothesis across training domains}
For \textit{optimal hypothesis across training domains condition}, we simply adopting the objective set forth by ERM: 
\begin{align}
\label{eq:emp_IRM}
     \min_{f} \; \sum_{i=1}^K \mathcal{L}\left(f,\mathbb{D}^{e_i}\right)
\end{align}

\subsection{Subspace Representation Alignment}
\paragraph{Subspace Modelling and Projection.} 
\label{sec:subspace_project_detail}
Our objective is to map samples $x$ from training domains with identical predictions $f(x) = m$ into a unified subspace, where $m\in \mathcal{M}=\{\hat{y}\mid \hat{y}=f(x), x\in\bigcup_{e\in\mathcal{E}_{tr}}\text{supp}\mathbb{P}^{e} \}$. Given that the cardinality of $\mathcal{M}$ can be exceedingly large, potentially equal to the total number of training samples if the output of $f(x)$ is unique for each sample, this makes the optimization process particularly challenging.

Drawing inspiration from the concept of prototypes \cite{snell2017prototypical}, we suggest representing $\mathcal{M}$ as a set of prototypes $\mathcal{M} = \{m_i\}_{i=1}^{M}$, where each $m_i$ is an element of $\mathcal{Z}$. Consequently, a sample $x$ is assigned to a subspace by selecting the nearest prototype $m_i$ i.e., $i=\underset{ i'}{\text{argmin }} \rho(g(x),m_{i'})$. Note that prototypes act as condensed representations of specific prediction outcomes. Consequently, samples assigned to the same prototype will receive the same prediction. Although this approach streamlines the subspace projection, it may lead to local optima as the mapping might favor a limited number of prototypes early in training \cite{vuong2023vector}. To mitigate this issue, we adopt a Wasserstein (WS) clustering approach \cite{vuong2023vector} to guide the mapping of latent features from each domain into the designated subspace more effectively.
We first endow a discrete distribution over the prototypes as $\mathbb{P}_{\mathcal{M},\pi}=\sum_{i=1}^{M}\pi_{i}\delta_{m_{i}}$
with the Dirac delta function $\delta$ and the weights $\pi\in\Delta_{M}= \{\pi'\geq \boldsymbol{0}: \Vert \pi'\Vert_1 =1\}$. 

Then we project each domain $\mathbb{P}^{e_i}$ in subspaces indexed by prototypes as follows:
\begin{equation}
\min_{\mathcal{M},\pi}\min_{g}\left \{\mathcal{L}_{P}=\sum_{i=1}^K\lambda\mathcal{W}_{\rho}\left(g\#\mathbb{P}^{e_i},\mathbb{P}_{\mathcal{M},\pi}\right)\right \},\label{eq:reconstruct_form_continuous}
\end{equation}
where:
\begin{itemize}
    \item Cost metric $\rho(z,m)=\frac{z^\top m}{\left \| z \right \|\left \| c \right \|}$ is the cosine similarity between the latent representation $z$ and the prototype $c$.
    \item  Wasserstein distance between source domain representation distribution and distribution over prototype $\mathbb{P}_{\mathcal{M},\pi}$:
\begin{align}
\mathcal{W}_{d}\left(g\#\mathbb{P}^{e_i}_{x},\mathbb{P}_{c,\pi}\right)
&=\mathcal{W}_{d}\left(\sum_{n=1}^{B}\frac{1}{B}g\left(x_{n}\right),\sum_{i=1}^{M}\pi_{i}\delta_{m_{i}}\right)\\
&=\frac{1}{B}\min_{\Gamma:\Gamma\#\left(g\#\mathbb{P}^{e_i}_{x}\right)=\mathbb{P}_{c,\pi}}\sum_{n=1}^{B}\rho\left(g\left(x_{n}\right),\Gamma\left(g\left(x_{n}\right)\right)\right)
\end{align}

Where $B$ is the batch size. This Wasserstein distance can be effectively compute by linear dynamic programming method, entropic dual form of optimal transport \citep{genevay2016stochastic} or Sinkhorn algorithm \cite{cuturi2013sinkhorn}.
\end{itemize}

\paragraph{Subspace Alignment Constraints}
Subspace alignment is achieved through a conditional adversarial training approach \cite{gan2016learning, li2018domain}. In this framework, the \textbf{subspace-conditional} domain discriminator $h_d$ aims to accurately predict the domain label ``$e_i$" based on the combined feature $[z,m]$, where $\{z=g(x), m=\Gamma(x)\}$. Concurrently, the objective for the representation function $g$ is to transform the input $x$ into a latent representation $z$ in such a way that $h_d$ is unable to determine the domain ``$e_i$" of $x$.  We employ the Gradient Reversal Layer (GRL) as introduced by\cite{ganin2016domain}, thereby simplifying the optimization process to:

\begin{equation}
    \min_{g, h_d} \left \{\mathcal{L}_{D}=-\sum_{i=1}^K\mathbb{E}_{x\sim\mathbb{D}^{e_i}}\left [ e_i\log h_d\left (\left [ \mathcal{R}\left ( g(x) \right ),m \right ]\right ) \right ] \right \}
\end{equation}

where $\mathcal{R}$ is gradient reversal layer.


\subsection*{Final objective}
Putting all together, we propose a joint optimization objective, which is given as  
\begin{equation}
\min_{\mathcal{M},\pi} \min_{g,h, h_d}  \left \{\mathcal{L}_{H}+\lambda_P\mathcal{L}_{P}+\lambda_D \mathcal{L}_{D}\right \},
\end{equation}
where $ \lambda_P$ is the subpsace projector hyper-parameter and $\lambda_D$ is the representation alignment hyper-parameter.

We highlight that SRA is most similar to DANN and CDANN. Like these methods, SRA utilizes $\mathcal{H}$-divergence for alignment; however, the key distinction lies in the alignment strategy: 
\begin{itemize}
    \item DANN aligns the entire domain representation, 
    \item CDANN aligns class-conditional representations,
    \item while SRA employs subspace-conditional alignment.
\end{itemize}

It is also important to note that the representation alignment hyperparameter $\lambda_D$ is kept the same for DANN, CDANN, and SRA in our experiments. As discussed in Theorem~\ref{theorem:single_tradeoff}, DANN and CDANN potentially violate necessary conditions, whereas SRA does not (Theorem~\ref{theorem:multi_bound}), leading to improved performance.

\section{Experimental Settings}
\label{apd:settings}


\paragraph{Metrics.} We adopt the training and evaluation protocol as in DomainBed benchmark \citep{gulrajani2020search}, including dataset splits, hyperparameter (HP) search, model selection on the validation set, and optimizer HP. To manage computational demands more efficiently, as suggested by \citep{cha2021swad}, we narrow our HP search space. Specifically, we use the Adam optimizer, as detailed in \citep{gulrajani2020search}, setting the learning rate to a default of $5e^{-5}$ and forgoing dropout and weight decay adjustments. The batch size is maintained at 32. For DomainNet, we run a total of 15,000 iterations, while for other datasets, we limit iterations to 5,000, deemed adequate for model convergence. Our method's unique parameters, including the regularization hyperparameters $(\lambda_P, \lambda_D)$, undergo optimization within the range of $[0.01, 0.1, 1.0]$, and the number of prototypes $\left | \mathcal{Z} \right |$ is fixed at 16 times the number of classes. It is worth noting that  while we conduct ablation study on PACS dataset, we utilize the number of prototypes $\left | \mathcal{Z} \right |$ is fixed at $16$ times the number of classes for all datasets.
SWAD-specific hyperparameters remain unaltered from their default settings. The evaluation frequency is set to 300 for all dataset.

Our code is anonymously published at \url{https://anonymous.4open.science/r/submisson-FCF0}.

\subsection{Datasets}
To evaluate the effectiveness of the proposed method, we utilize five
datasets: PACS~\citep{li2017deeper}, VLCS~\citep{torralba2011unbiased},
 Office-Home~\citep{venkateswara2017deep}, Terra Incognita~\citep{beery2018recognition} and DomainNet~\citep{peng2019moment} which are the common DG benchmarks with multi-source domains.
\begin{itemize}
    \item \textbf{PACS}~\citep{li2017deeper}: 9991 images of seven classes in total, over four domains:Art\_painting (A), Cartoon (C), Sketches (S), and Photo (P). 
    
    \item \textbf{VLCS}~\citep{torralba2011unbiased}: five classes over four domains with a total of 10729 samples. The domains are defined by four image origins, i.e., images were taken from the PASCAL VOC 2007 (V), LabelMe (L), Caltech (C) and Sun (S) datasets.

    \item \textbf{Office-Home}~\citep{venkateswara2017deep}: 65 categories of 15,500 daily objects from 4 domains: Art, Clipart, Product (vendor website with white-background) and Real-World (real-object collected from regular cameras).
    \item \textbf{Terra Incognita}~\citep{beery2018recognition} includes 24,788 wild photographs of dimension (3, 224, 224) with 10
animals, over 4 camera-trap domains L100, L38, L43 and L46. This dataset contains photographs of wild animals taken by camera traps; camera trap locations are different across 
domains. 
    \item  \textbf{DomainNet}~\citep{peng2019moment} contains 596,006 images of dimension (3, 224, 224) and 345 classes, over
6 domains clipart, infograph, painting, quickdraw, real and sketch. This is the biggest
dataset in terms of the number of samples and classes.
\end{itemize}

\end{document}